\documentclass[a4paper,12pt]{article}

\usepackage[utf8]{inputenc}
\usepackage[T1]{fontenc}

\usepackage[english]{babel}

\usepackage{mathptmx}

\usepackage{graphicx}
\usepackage[labelsep=period,labelfont=bf,font=small]{caption}

\usepackage[text={13cm,22cm},centering]{geometry}

\usepackage{natbib}
\bibpunct[: ]{(}{)}{,}{a}{,}{,}

\usepackage{amsmath}
\usepackage{amssymb}
\usepackage{amsthm}
\newtheoremstyle{break}%
{}{}%
{\itshape}{}%
{\bfseries}{}%
{\newline}%
{}
\newtheoremstyle{breakdf}%
{}{}%
{\rm}{}%
{\bfseries}{}%
{\newline}%
{}
\theoremstyle{break}

\newtheorem*{corollary}{Corollary}
\newtheorem{theorem}{Theorem}
\theoremstyle{breakdf}
\newtheorem{algo}{Algorithm}
\newcommand{\Prob}{\textnormal{Prob}}
\renewcommand{\vec}[1]{\mathbf{#1}}
\DeclareMathOperator{\intr}{int}
\DeclareMathOperator{\bd}{bd}

\usepackage{sectsty}
\allsectionsfont{\normalsize\bfseries}
\usepackage{enumitem}
\setlist{leftmargin=*,itemsep=0pt}

\begin{document}

\title{\Large Stable variation in multidimensional competition\thanks{This manuscript is a pre-final version of a paper published in A.~Breitbarth, M.~Bou\-zouita, L.~Danckaert \& M.~Farasyn (eds.), \emph{The determinants of diachronic stability} (Amsterdam, John Benjamins, 2019, pp.~263--290, https://doi.org/10.1075/la.254.11kau). For citations involving page numbers or verbatim quotations, please refer to the publisher's version. Copyright \copyright{} John Benjamins Publishing Company. The publisher should be contacted for permission to re-use or reprint the material in any form.}}
\author{\normalsize Henri Kauhanen\\\normalsize\emph{The University of Manchester}}
\date{\normalsize 2019}
\maketitle


\begin{abstract}
The Fundamental Theorem of Language Change \citep{Yan2000} implies the impossibility of stable variation in the Variational Learning framework, but only in the special case where two, and not more, grammatical variants compete. Introducing the notion of an advantage matrix, I generalize Variational Learning to situations where the learner receives input generated by more than two grammars, and show that diachronically stable variation is an intrinsic feature of several types of such multiple-grammar systems. This invites experimentalists to take the possibility of stable variation seriously and identifies one possible place where to look for it: situations of complex language contact.
\end{abstract}

\section{Variation, learning and diachronic stability}\label{sec:introduction}

Since its introduction in a series of publications by Yang in the early noughties \citep{Yan1999,Yan2000,Yan2002b,Yan2002,Yan2004}, the Variational Learner has stirred much interest among those working in the field of language variation and change: given its inherently probabilistic nature, the Variational Learning paradigm successfully formalizes many aspects of the competing grammars framework \citep{Kro1994}, in which the simultaneous existence of a number of grammatical options in the mind of a speaker is taken for granted. As far as \emph{change} is concerned, however, this intra-speaker existence of multiple grammars has been considered \emph{diachronically unstable}, in the sense that over iterated generational learning interactions, grammar competition leads, ultimately, to a stable state of dominance by some single grammar. This mathematical fact, formulated as the Fundamental Theorem of Language Change by \citet{Yan2000}, dovetails with the theoretico-empirical claim that all morphosyntactic variation between two forms competing for a single function results, over time, in either the extinction of one form, or a functional specialization of the two forms by which the competition is escaped \citep{Kro1994,Wal2016} -- in either case, diachronically stable variation between two values of a single variable is thought to be impossible because of a general cognitively motivated blocking effect that militates against stable doublets \citep{Aro1976}.

In this paper, I wish to draw attention to the fact that Variational Learning only predicts this outcome in the case where two, and not more than two, variants compete in a speaker population. An analysis of both the classical Variational Learner and its parametrically constrained variation, the Naive Parameter Learner, reveals that in the general case -- when more than two grammars compete -- the situation is strikingly different. The Fundamental Theorem gives way to more complicated, even non-monotonic trajectories of change; to bifurcations; and, in many cases, to truly stable variation in which the competing variants do not (or need not) specialize functionally. Since language learners need to set the values of multiple parameters and hence make a choice in a high-dimensional space of possible grammars, these results question whether the Variational Learner can, in fact, explain the (purported) non-occurrence of stable variation. On the other hand, the results invite experimentalists to consider the possibility that when more than two variants come to compete, stable variation may in fact be predicted by general human learning mechanisms (assuming, \emph{ex hypothesi}, that the reinforcement learning algorithm at the heart of the Variational Learner carries psychological realism).

To begin, it is incumbent on us to make the relevant notions of variation and stability as precise as possible. Any system capable of change is a dynamical system whose behaviour may be modelled using a set of difference equations -- if the time variable is taken as discrete -- or a set of differential equations -- if time is considered continuous. The choice of one or the other description is largely arbitrary; in this paper, I will stick to discrete time, but all the results are valid for a continuous-time description as well (by letting the inter-generational time step tend to zero and examining the resulting differential equations). I then define a \emph{language system} to be a probability distribution $\vec p = (p_1, \dots , p_n)$ over a finite set of possible grammars $G_1,\dots , G_n$, together with a set of difference equations
\begin{equation}\label{eq:gen-diff-eq}
  p_i' = f_i(\vec p) \quad (i = 1,\dots ,n)
\end{equation}
which define the system's dynamics. Here, $p_i'$ is the \emph{successor} of $p_i$; in other words, $p_i'$ is the value of the $i$th variable at time $t+1$ given that the state of the entire system at time $t$ was $\vec p = (p_1,\dots ,p_n)$. The functions $f_i$ are, in the general case, real-valued functions; they assume some concrete form as soon as concrete assumptions are made about learning, linguistic interaction, the existence of a critical period, and so on. The probabilities $p_i$ themselves, $0 \leq p_i \leq 1$, describe the probability of use of the different competing grammars, in the usual sense: in a sequence of $k$ utterances, roughly $p_ik$ utterances will be produced by grammar $G_i$ if $k$ is large. These probabilities may be taken to describe either a single individual or an entire community of speakers: clearly, both individual and community-level probabilities may change over time, but the corresponding functions $f_i$ in \eqref{eq:gen-diff-eq} may be rather different in the two cases. In what follows, I will always take $p_i$ to refer to community-level probabilities and will denote probabilities at the level of individuals with corresponding Greek letters, $\pi_i$.

Taking the $p_i$ as community-level probabilities, then, let us proceed to define the notions of variation and stability on the level of speech communities. Intuitively, variation exists if at least two grammars are used with non-zero probability. It then makes sense to define a \emph{state of variation} as a probability state $\vec p = (p_1, \dots ,p_n)$ which satisfies $p_i < 1$ for all $i$, for it is precisely under this condition that no single grammar gets to claim all of the available probability mass. Defining the concomitant notion of diachronic stability is a bit trickier, and I shall begin by presenting a physical analogue.

Consider a non-ideal pendulum (Figure \ref{fig:pendula}A). By non-ideal, I mean to imply that we are \emph{not} excluding frictional forces by way of idealization. Such a pendulum is also known as a damped pendulum, and the defining characteristic of its dynamics is the existence of a \emph{rest point} directly below the point of attachment: if the pendulum is ever found in this position, it will not move, barring application of an external force.\footnote{In the corresponding mathematical description, a rest point is identified as a state $\vec x$ which satisfies $\vec x' = \vec x$ or equivalently $\vec x' - \vec x = \vec 0$, that is, as a zero-change state. In the vast literature on dynamical systems, rest points are also known as rest states, fixed points, equilibria, and steady states. The last term, sometimes encountered in discussions of language change, is somewhat unfortunate because of the semantic similarity of the pre-theoretical terms `steady' and `stable' -- as we will see presently, not all steady states are stable, in the technical sense.} Moreover, if the pendulum is set in motion from some other initial state, it will ultimately come to a halt at this rest point after a period of diminishing oscillation. Such a rest point is said to be \emph{asymptotically stable}. More precisely, a rest point $\vec x$ in the state space of a dynamical system is asymptotically stable if a neighbourhood of states around $\vec x$ exists such that all trajectories from this neighbourhood converge to $\vec x$ as time tends to infinity.

\begin{figure}
  \centering
  \includegraphics[width=\textwidth]{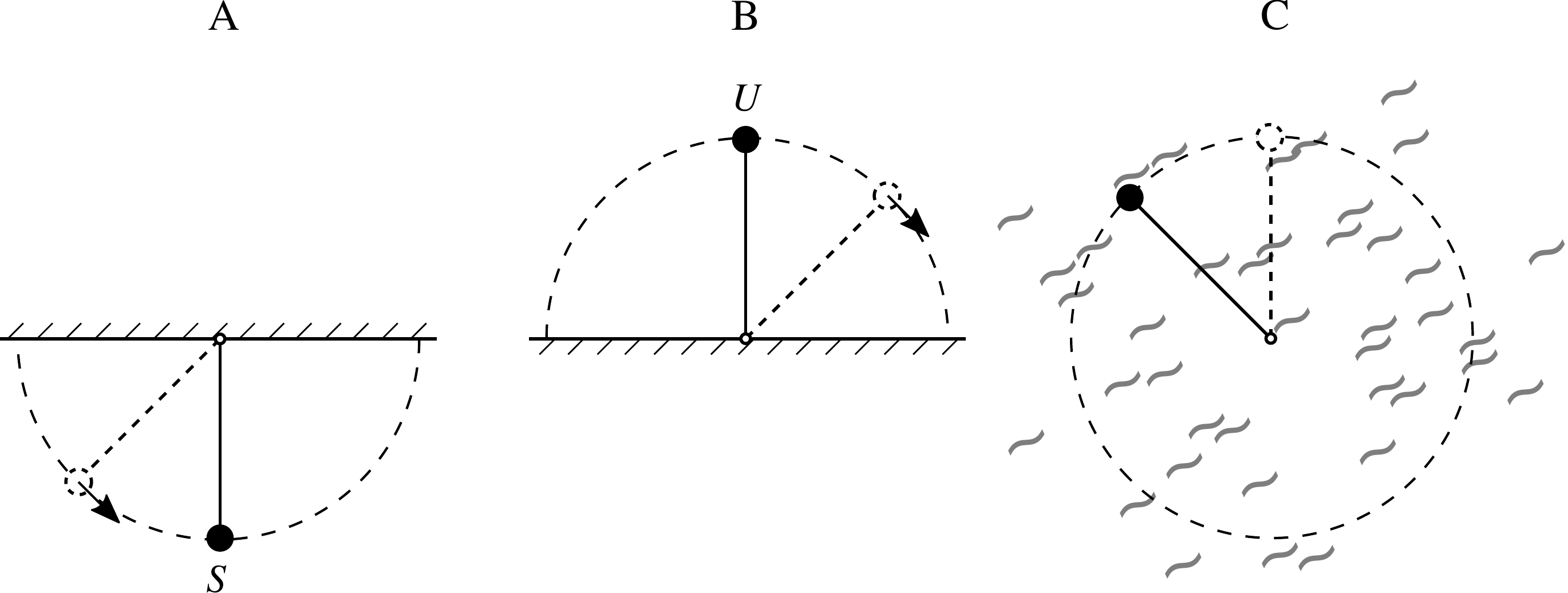}
  \caption{A damped pendulum (A), an inverted pendulum (B) and a ``goo pendulum'' (C). $S$: stable rest point, $U$: unstable rest point.}\label{fig:pendula}
\end{figure}

Now consider the inverted pendulum of Figure \ref{fig:pendula}B. This pendulum, too, has a rest point, now directly above the point of attachment. Theoretically, if it were possible to balance the pendulum with infinite precision at this rest point, it would not move, since the horizontal component of the sum of the forces acting on the pendulum is zero at this point (we assume the pendulum is fixed to a stiff rod). Even a slight disturbance to the inverted pendulum will, however, nudge it away from the rest point. Such a rest point is \emph{unstable}, since all trajectories from \emph{any} local neighbourhood around the rest point take the system state away from the rest point.

Finally, consider the ``goo pendulum'' of Figure \ref{fig:pendula}C. Here the pendulum is submerged in a hypothetical goo of infinite viscosity which supports the pendulum but allows its movement when a suitable external force is applied (for a physically realistic approximation, we may think of a low-mass pendulum, such as a needle, submerged in a high-viscosity fluid such as honey). This pendulum will not move from any initial condition. Every possible position of the pendulum is a rest point, and they are all neither asymptotically stable nor unstable. The characteristic behaviour of these \emph{non-asymptotically stable} states is that, given a perturbation, the system will move to a different, close-by point, but is not ``actively'' repelled by the rest point nor attracted back to it.

These notions translate directly into our framework of language systems and may now be used to explicate the idea of stable variation. I define a \emph{state of stable variation} to be a probability vector $\vec p = (p_1, \dots ,p_n)$ satisfying the following three conditions simultaneously:
\begin{enumerate}
  \item $\vec p$ is a state of variation ($p_i < 1$ for all $i$)
  \item $\vec p$ is a rest point ($p_i' - p_i = 0$ for all $i$)
  \item $\vec p$ is asymptotically stable
\end{enumerate}
I do not include non-asymptotically stable rest points in this definition since, as per the above discussion, they are not resilient to perturbations. Crucially, given that real-life systems always contain a source of noise, which we may think of as a perturbation to the state of a deterministic system such as \eqref{eq:gen-diff-eq}, such states do not count as truly stable.

\section{Two grammars}\label{sec:2D}

With these notions in hand we may proceed to a formal study of variation and stability in the Variational Learning framework, beginning with a summary restatement of the already familiar two-grammar case.

In \citet{Yan2000}, language change is reduced to language acquisition by assuming that language learners employ a specific learning strategy, the linear reward--penalty (henceforth, LRP) learning algorithm originating in Bush and Mosteller's \citeyearpar{BusMos1955} early work on reinforcement learning and most usefully synthesized by \citet{NarTha1989}. This allows one to close the population-dynamical equations \eqref{eq:gen-diff-eq}. Specifically, assume the learner needs to make a decision between two grammars $G_1$ and $G_2$ which are used in the community with probabilities $p_1$ and $p_2$. Writing $\pi_1$ and $\pi_2$ for the learner's hypothesis (i.e.~$\pi_i$ is the probability with which the learner himself employs $G_i$), the LRP algorithm assumes the following form:
\begin{algo}[LRP, $n=2$; \protect\citealp{NarTha1989}: 110--111]
  \leavevmode\vspace{-\baselineskip}
\begin{enumerate}
  \item Let $\pi_1 = \pi_2 = 1/2$ initially.
  \item Present an input token (sentence) $x$ to the learner. This is generated by $G_1$ with probability $p_1$ and by $G_2$ with probability $p_2$.
  \item Learner picks grammar $G_i$ with probability $\pi_i$.
  \item Suppose the learner picked $G_1$.
    \begin{enumerate}[label=\alph*.]
      \item If $G_1$ parses $x$, the learner increases $\pi_1$ by a small amount and decreases $\pi_2$ by a small amount. Concretely, $\pi_1$ is replaced with $\pi_1 + \gamma(1-\pi_1)$, where $\gamma$ is a small positive number (the \emph{learning rate}), whilst $\pi_2$ is replaced with $(1-\gamma )\pi_2$.
      \item Conversely, if $G_1$ does not parse $x$, the learner decreases $\pi_1$ and increases $\pi_2$. Concretely, $\pi_1$ is replaced with $(1-\gamma )\pi_1$, whilst $\pi_2$ is replaced with $\pi_2 + \gamma (1 - \pi_2)$.
    \end{enumerate}
  \item (If the learner picked $G_2$ instead, execute the previous step with labels $1$ and $2$ interchanged.)
  \item Steps 2--5 are repeated for $T$ input tokens.
\end{enumerate}
\end{algo}
\noindent Thus, during learning, the probabilities $\pi_i$ change in response to the two grammars' success in parsing input generated from the community-level distribution $\vec p = (p_1, p_2)$. For simplicity, the latter is assumed to stay constant for the duration of learning; in learning-theoretic terminology, the learner's environment is a \emph{stationary random environment} \citep{NarTha1989}.

If either $p_1 = 1$ or $p_2 = 1$, then one of the grammars succeeds in parsing any possible input token the learner may encounter. It then follows that in such a case of a homogeneous community, the learner's hypothesis tends to the population state with growing $T$ and the unique target grammar is learnable according to a probabilistic variant of Gold's \citeyearpar{Gol1967} learnability criterion (cf.~\citealp[354]{Niy2002}). If the population state $\vec p = (p_1, p_2)$ is mixed, i.e.~a state of variation, the learner exhibits more interesting behaviour.

Let $\hat \pi_i$ denote the value of $\pi_i$ at the end of learning (at $T$ learning steps), and assume that $T$ is large and that the learning rate $\gamma$ is small. Such a learner shall be called \emph{reliable},\footnote{All results in this paper pertain to systems of reliable learners. The stochastic effects of unreliable learning -- short critical periods or large (``high-temperature'') learning rates -- are underinvestigated in the literature but must be set aside here.} and it can be shown \citep[111--112]{NarTha1989} that, for a reliable learner,
\begin{equation}\label{eq:2D-LRP-approximation}
  \hat \pi_1 \approx \frac{c_2}{c_1 + c_2}\quad \textnormal{and}\quad \hat\pi_2 \approx \frac{c_1}{c_1 + c_2},
\end{equation}
where $c_i$ is the \emph{penalty probability} of grammar $G_i$:
\begin{equation}
  c_i = \Prob (x : G_i\ \textnormal{does not parse}\ x).
\end{equation}
The penalty probabilities are easily determined: we may write $c_1 = a_{2}p_2$ and $c_2 = a_{1}p_1$, where $a_{2}$ is the probability of a sentence parsed by $G_2$ but not by $G_1$, and vice versa for $a_{1}$. Following \citet{Yan2000}, I will call $a_{1}$ the \emph{advantage} of $G_1$ and $a_{2}$ the advantage of $G_2$ (Figure \ref{fig:2D-penalties}).

\begin{figure}
  \centering
  \includegraphics[scale=0.52]{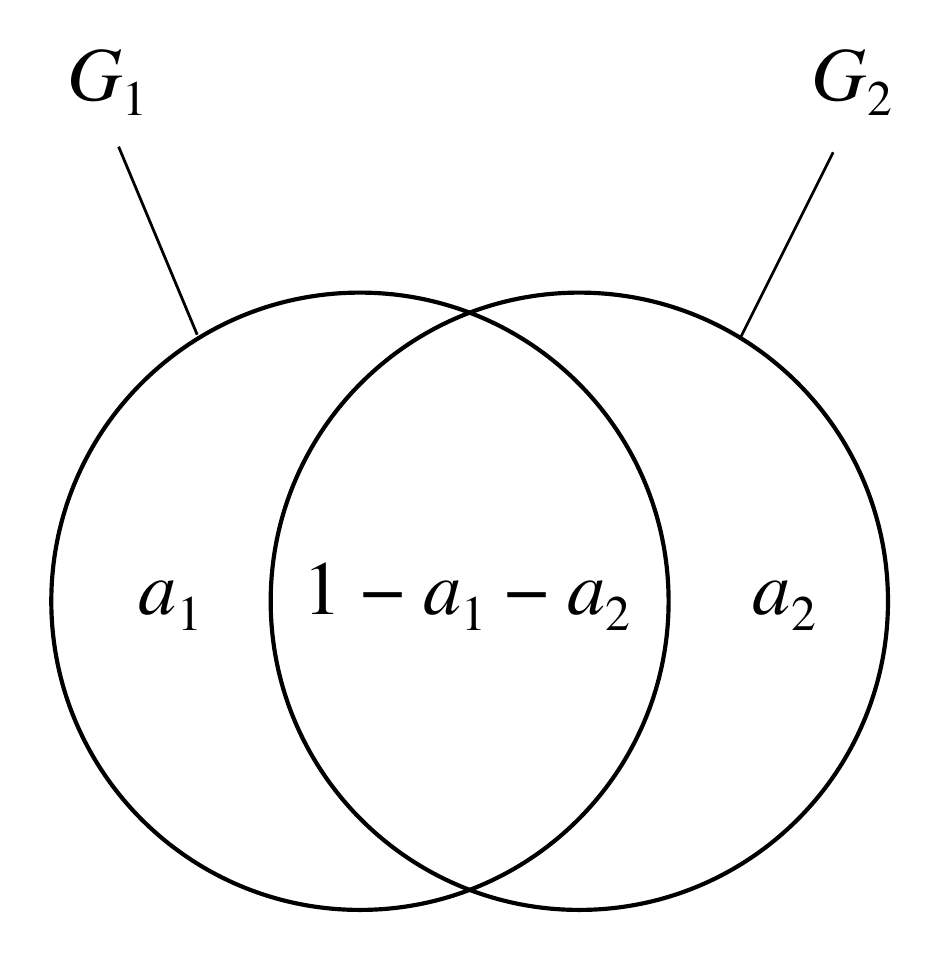}
  \caption{The classical two-grammar setting (after \citealp[238]{Yan2000}, Figure 2). This Venn diagram illustrates all sentences parsed by either grammar; $a_1$ is the probability of a sentence uniquely parsed by $G_1$ and $a_2$ the probability of a sentence uniquely parsed by its competitor $G_2$.}\label{fig:2D-penalties}
\end{figure}

If learners are now arranged in a sequence of non-overlapping generations, the output of generation $t$ feeding as input to the learning process of generation $t+1$, we have the population-level difference equations
\begin{equation}
  p_1' = \frac{a_{1}p_1}{a_{1}p_1 + a_{2}p_2}\quad\textnormal{and}\quad
  p_2' = \frac{a_{2}p_2}{a_{1}p_1 + a_{2}p_2}.
\end{equation}
Bearing in mind that $p_1 + p_2 = 1$, it suffices to work with the single equation
\begin{equation}\label{eq:2grammars-diff-eq}
  p_1' = \frac{a_{1}p_1}{a_{1}p_1 + a_{2}(1-p_1)}.
\end{equation}
The inter-generational increment in $p_1$ is given by $p_1' - p_1$, which by simple algebra is found to equal
\begin{equation}\label{eq:2grammars-intergen}
  p_1' - p_1 = \frac{(a_{1} - a_{2})(1-p_1)p_1}{a_{1}p_1 + a_{2}(1-p_1)}.
\end{equation}
Figuring out the rest points of this system is now an easy task: from \eqref{eq:2grammars-intergen} it is readily seen that $p_1' - p_1 = 0$ if and only if (1) $p_1 = 0$, (2) $p_1 = 1$ or (3) $a_{1} = a_{2}$. Assume first that $a_{1} > a_{2}$. Then the sign of $p_1' - p_1$ is always strictly positive, which means that $p_1$ always grows, no matter what the state $\vec p = (p_1, p_2)$. Hence, the state $(1,0)$ is asymptotically stable and the state $(0,1)$ unstable. With this ordering of the two advantage parameters, $G_1$ will drive $G_2$ out in diachrony, no matter what the initial state of the system. For $a_{1} < a_{2}$, the reverse state of affairs obtains: $(1,0)$ is unstable and $(0,1)$ stable. Now $G_2$ is the winner. Finally, if $a_{1} = a_{2}$, then the rate of change of $p_1$ (and, by necessity, of $p_2$) is zero in \emph{every} possible state $\vec p = (p_1, p_2)$. The state space is filled with an infinity of non-asymptotically stable rest points, and the system resembles the goo pendulum of Figure \ref{fig:pendula}C. With this reasoning, we have proved the following two results:
\begin{theorem}[Fundamental Theorem of Language Change; \protect\citealp{Yan2000}: 239]
  Suppose learners are reliable. Then, in a two-grammar system, $G_1$ wins in diachrony if $a_1 > a_2$, and $G_2$ wins if $a_1 < a_2$.\qed
\end{theorem}
\begin{theorem}
  No two-grammar system of reliable learners admits stable variation.\qed
\end{theorem}
\noindent Figure \ref{fig:2D-trajectory} illustrates a typical trajectory in a system of two grammars with unequal advantages. The grammar with the greater advantage ousts its competitor both in the case of theoretically perfectly reliable learners (equation \ref{eq:2grammars-diff-eq}) and in the case of learners who receive a finite but large sample of primary linguistic data.

\begin{figure}
  \centering
  \includegraphics[width=\textwidth]{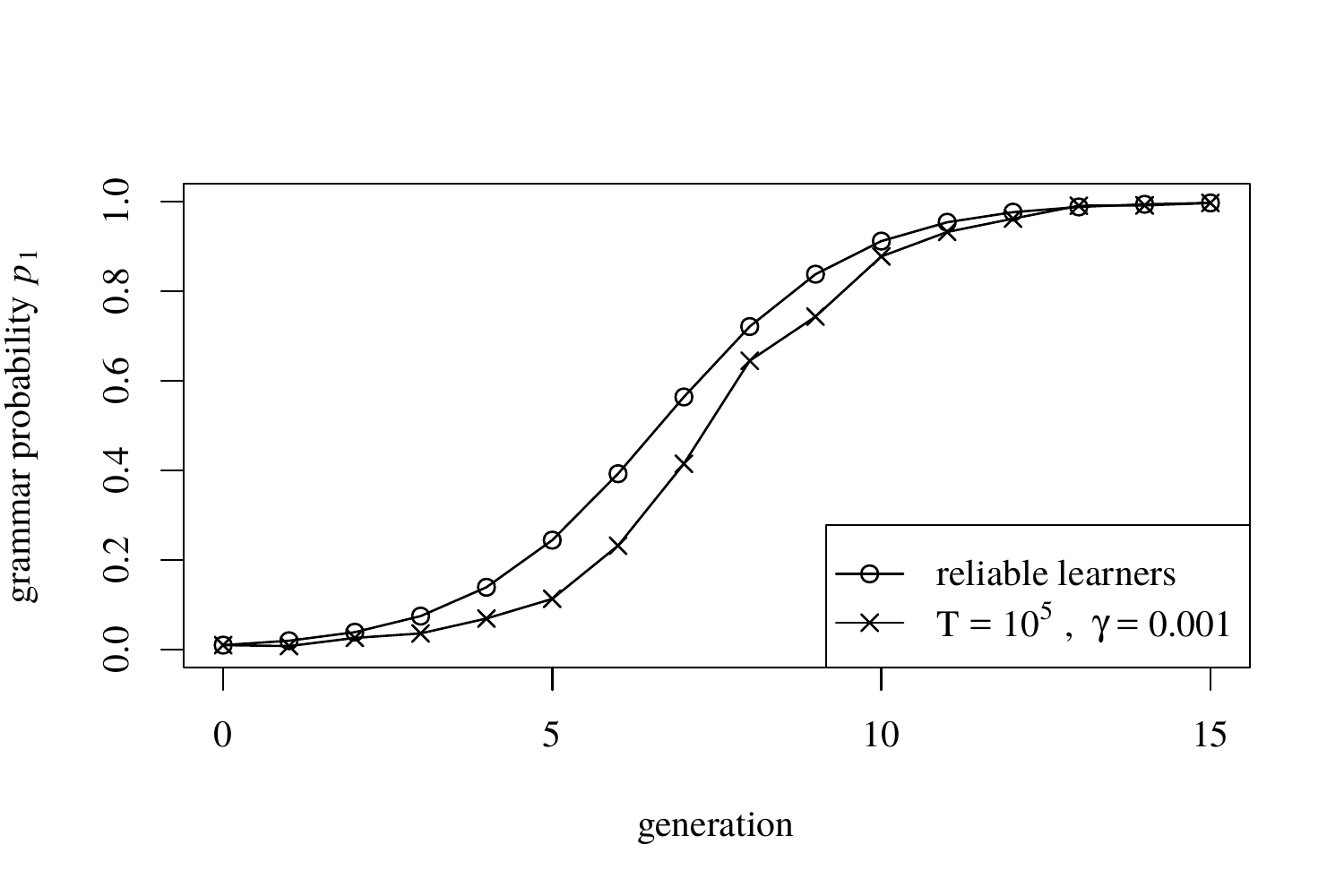}
  \caption{Time evolution of a two-grammar system with $a_1 = 0.2$ and $a_2 = 0.1$, from initial state $(p_1, p_2) = (0.01, 0.99)$, for both theoretically perfectly reliable learners (circles, equation \ref{eq:2grammars-diff-eq}) and for large-sample learners (crosses, from computer simulation; only one realization of the stochastic process shown).}\label{fig:2D-trajectory}
\end{figure}

\section{Advantage matrices and the cyclical balance criterion}\label{sec:3D}

It is not immediately obvious how, or whether, these results generalize to situations where learners are exposed to input from more than two grammars. In fact, extending the model definition itself to such more general cases turns out to be nontrivial. The main difficulty lies in expressing the penalty probabilities $c_i$, which with an increasing number of competing variants assume an increasingly complicated form. This is because in the general case of $n$ competing grammars one has to consider the relative (pairwise) advantages between any two distinct grammars, the number of these advantage relations being $n(n-1) = n^2 - n$ and hence growing superlinearly with $n$.

\begin{figure}
  \centering
  \includegraphics[scale=0.55]{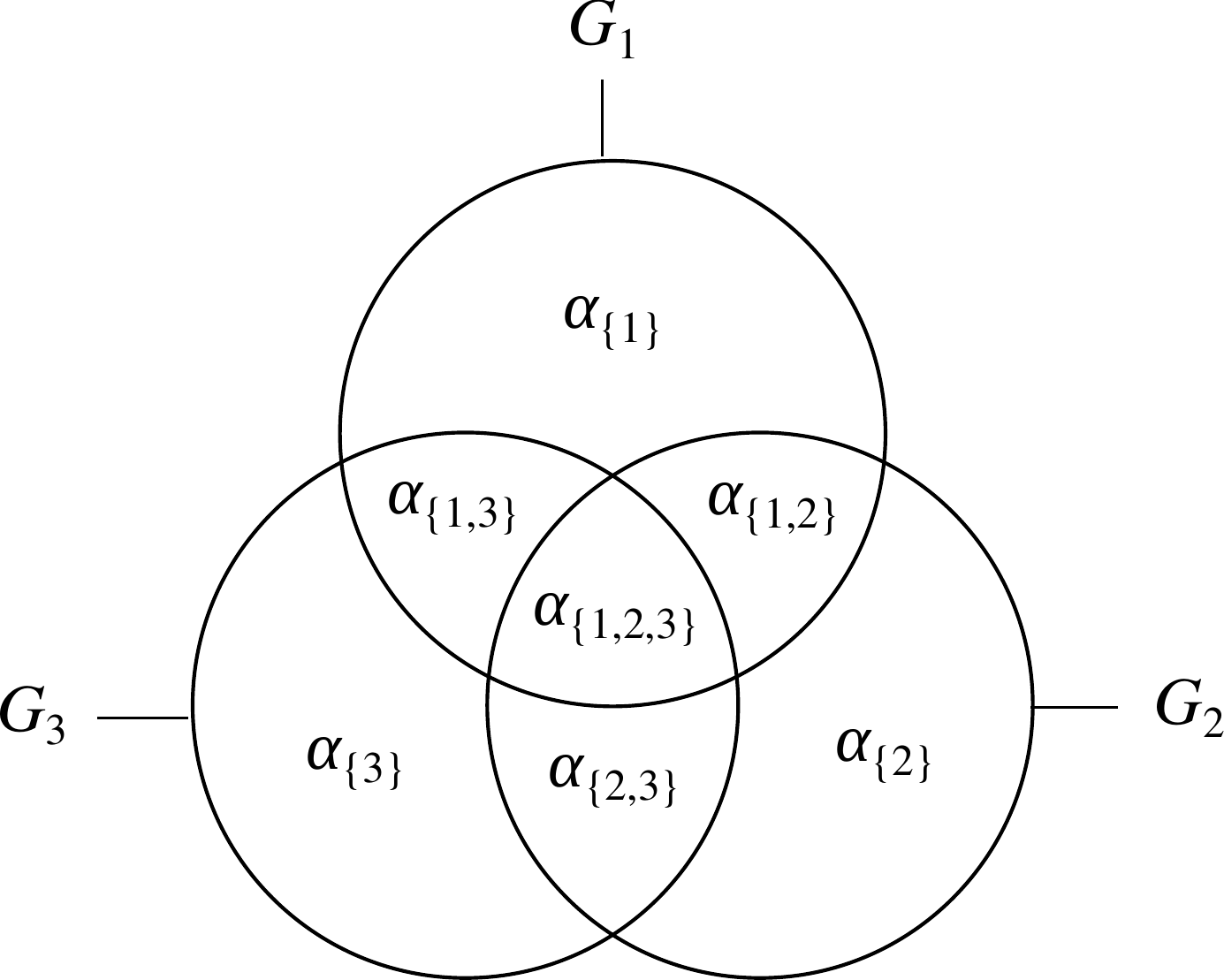}
  \caption{Venn diagram illustrating the general three-grammar case. Here, $\alpha_I$ gives the probability of a sentence parsed by all and only the grammars $G_i$ for which $i \in I$, where $I$ indexes the subsets of $\{1,2,3\}$.}\label{fig:3D-venn}
\end{figure}

In the three-grammar case ($n=3$), the situation is as depicted in Figure \ref{fig:3D-venn}. Each grammar potentially generates sentences which are only parsed by that grammar itself. However, the possibility now arises that two of the three grammars jointly generate something not parsed by the third grammar. Using the symbolism of Figure \ref{fig:3D-venn}, we find that the penalty probability for grammar $G_1$ in this more general three-grammar situation may be expressed as
\begin{equation}
  \begin{split}
    c_1 &= \alpha_{\{2\}}p_2 + \alpha_{\{3\}}p_3 + \alpha_{\{2,3\}}(p_2 + p_3) \\
    &= (\alpha_{\{2\}} + \alpha_{\{2,3\}})p_2 + (\alpha_{\{3\}} + \alpha_{\{2,3\}})p_3.
  \end{split}
\end{equation}
If we now write $a_{12} = \alpha_{\{2\}} + \alpha_{\{2,3\}}$ and $a_{13} = \alpha_{\{3\}} + \alpha_{\{2,3\}}$, we see that $a_{12}$ gives the \emph{relative advantage} of $G_2$ over $G_1$ and $a_{13}$ the relative advantage of $G_3$ over $G_1$. Proceeding analogously to derive the penalty probabilities $c_2$ and $c_3$, one finds
\begin{equation}\label{eq:3D-penalties}
  \left\{
    \begin{aligned}
      c_1 &= a_{12}p_2 + a_{13}p_3 \\
      c_2 &= a_{21}p_1 + a_{23}p_3 \\
      c_3 &= a_{31}p_1 + a_{32}p_2
    \end{aligned}
  \right.
\end{equation}
where each $a_{ij}$ thus gives the probability of a sentence which is parsed by $G_j$ but not by $G_i$. It is these relative advantages $a_{ij}$ that determine the system's dynamics, and consequently it will be useful to collect them in a matrix,
\begin{equation}
  \vec A = [a_{ij}] = \begin{bmatrix}
    0 & a_{12} & a_{13} \\
    a_{21} & 0 & a_{23} \\
    a_{31} & a_{32} & 0
  \end{bmatrix}
\end{equation}
where the diagonal is zero since obviously $a_{ii} = 0$ for any $i$. In what follows, I will refer to such a matrix as an \emph{advantage matrix}. It is possible, with greater technical difficulty, to generalize this procedure for arbitrary $n$, and many of the results to follow carry over to the general case. Here, I restrict my attention to three grammars in the interest of readability.

Not every square matrix of real numbers is a valid advantage matrix. As already mentioned, the diagonal is necessarily zero, since no grammar both parses and does not parse one and the same sentence. Furthermore, from Figure \ref{fig:3D-venn}, we note that the $\alpha$ quantities must all sum to unity, since the event represented by their union is ``a sentence is produced which some grammar parses''. In three dimensions, this corresponds to the requirement
\begin{equation}
  \alpha_{\{1\}} + \alpha_{\{2\}} + \alpha_{\{3\}} + \alpha_{\{1,2\}} + \alpha_{\{1,3\}} + \alpha_{\{2,3\}} + \alpha_{\{1,2,3\}}  = 1.
\end{equation}
Rearranging the terms on the left hand side, we obtain
\begin{equation}\label{eq:cycbal-1}
  a_{21} + a_{32} + a_{13} + \alpha_{\{1,2,3\}}= 1.
\end{equation}
On the other hand, arranging the $\alpha$ terms differently, we have
\begin{equation}\label{eq:cycbal-2}
  a_{31} + a_{23} + a_{12} + \alpha_{\{1,2,3\}} = 1.
\end{equation}
From \eqref{eq:cycbal-1} and \eqref{eq:cycbal-2},
\begin{equation}
  a_{21} + a_{32} + a_{13} = a_{31} + a_{23} + a_{12}
\end{equation}
or
\begin{equation}
  (a_{21} - a_{12}) + (a_{32} - a_{23}) + (a_{13} - a_{31}) = 0.
\end{equation}
Writing $\delta_{ij} = a_{ji} - a_{ij}$, we have
\begin{equation}\label{eq:cyclical-balance-criterion}
  \delta_{12} + \delta_{23} + \delta_{31} = 0
\end{equation}
which I will refer to as the \emph{cyclical balance criterion}. The advantage matrix of any 3-grammar system, then, has to satisfy this criterion.

Within the remit of the cyclical balance criterion, many qualitatively different kinds of advantage matrix are possible. In particular, it is possible for some of the advantage quantities $a_{ij}$ to equal zero -- this will be the case if inclusion (subset--superset) relations exist among the competing grammars, in the sense that one grammar parses everything that another does. In what follows, I will however usually assume that $a_{ij} > 0$ for all $i$ and $j$ with $i\neq j$, and will say that an advantage matrix satisfying this condition is \emph{proper}. Assuming advantage matrices to be proper thus delimits the class of formal systems studied to some extent; the benefit of making this assumption is that it makes available a useful learning-theoretic approximation which is not available in the improper case, as we will shortly see. Without this approximation, the improper cases need to be studied separately, on a case-by-case basis.

\section{Dynamics: general results}\label{sec:dynamics}

With the penalty probabilities \eqref{eq:3D-penalties} in hand, we may now proceed to study the dynamics of the three-grammar case. The general form of the LRP algorithm reads as follows:
\begin{algo}[LRP; \protect\citealp{NarTha1989}: 116--117]\label{alg:LRP-nD}
  \leavevmode\vspace{-\baselineskip}
  \begin{enumerate}
    \item Let $\pi_i = 1/n$ initially.
    \item Present an input token (sentence) $x$ to the learner. This is generated by $G_i$ with probability $p_i$.
    \item Learner picks grammar $G_i$ with probability $\pi_i$.
    \item Suppose learner picked $G_k$.
      \begin{enumerate}[label=\alph*.]
        \item If $G_k$ parses $x$, learner replaces $\pi_k$ with $\pi_k + \gamma (1-\pi_k)$, with learning rate $\gamma$, and $\pi_j$ with $(1-\gamma )\pi_j$, $j\neq k$.
        \item If $G_k$ does not parse $x$, learner replaces $\pi_k$ with $(1-\gamma )\pi_k$ and $\pi_j$ with $\frac{\gamma}{n-1} + (1-\gamma )\pi_j$, $j\neq k$.
      \end{enumerate}
    \item Steps 2--4 are repeated for $T$ input tokens.
  \end{enumerate}
\end{algo}
\noindent Assuming reliable learners (large $T$, small $\gamma$), \citet[117]{NarTha1989} show that the following approximation holds for the learner's hypothesis at the end of the learning cycle:\footnote{If all the penalty probabilities are strictly positive, $c_i > 0$ for all $i$, then this slightly unwieldy formula reduces to the more aesthetic $\hat\pi_i \approx c_i^{-1} / \sum_j c_j^{-1}$ upon division of both the numerator and the denominator by $\prod_i c_i$. \citet{NarTha1989} limit their discussion to this case.}
\begin{equation}\label{eq:nD-LRP-approximation}
  \hat\pi_i \approx \frac{\prod_{j\neq i} c_j}{\sum_j\prod_{k\neq j} c_j}.
\end{equation}
Assuming non-overlapping generations of such learners thus yields the diachronic difference equation
\begin{equation}\label{eq:nD-diff-eq}
  p_i' = \frac{\prod_{j\neq i} c_j}{\sum_j\prod_{k\neq j} c_j}.
\end{equation}
In particular, in three dimensions one has
\begin{equation}\label{eq:3D-diff-eq}
  \left\{
    \begin{aligned}
      p_1' &= \frac{c_2c_3}{c_2c_3 + c_1c_3 + c_1c_2}\\
      p_2' &= \frac{c_1c_3}{c_2c_3 + c_1c_3 + c_1c_2}\\
      p_3' &= \frac{c_1c_2}{c_2c_3 + c_1c_3 + c_1c_2}
    \end{aligned}
  \right.
\end{equation}
where, it bears stressing, each penalty $c_i$ is itself a function of the system state $\vec p = (p_1, p_2, p_3)$, leading to a nonlinear equation. For this to be well-defined, mathematically speaking, we need to check that the denominators never equal zero. This is guaranteed for all proper advantage matrices:
\begin{theorem}\label{thm:zero-penalties}
  For a proper advantage matrix, $c_i = 0$ if and only if $p_i = 1$.
\end{theorem}
\begin{proof}
  Since $\vec A$ is proper, $c_i = \sum_{j\neq i} a_{ij} p_j = 0$ if and only if $p_j = 0$ for all $j\neq i$. But since $\vec p$ is a probability distribution, the latter occurs if and only if $p_i = 1$.
\end{proof}
\begin{corollary}
  Given a proper advantage matrix, it is never possible for two penalty probabilities $c_i$ and $c_j$, $i\neq j$, to equal zero at the same time. Consequently, the denominators in \eqref{eq:3D-diff-eq} are never zero.\qed
\end{corollary}
\noindent The learning-theoretic approximation \eqref{eq:nD-LRP-approximation} therefore leads to a well-defined inter-generational (diachronic) dynamical system whenever advantages are proper (as pointed out in the preceding discussion, the improper cases would need to be studied separately, a task which I set aside in the present paper).

As the $p_i$ are probabilities, the system \eqref{eq:3D-diff-eq} is defined on the $3$-dimensional \emph{simplex}
\begin{equation}
  S_3 = \{\vec p = (p_1, p_2, p_3) : 0 \leq p_1, p_2, p_3 \leq 1\ \textnormal{and}\ p_1 + p_2 + p_3 = 1\}.
\end{equation}
This set may be partitioned into the \emph{interior}
\begin{equation}
  \intr S_3 = \{\vec p \in S_3 : 0 < p_1, p_2, p_3 < 1\}
\end{equation}
and the \emph{boundary}
\begin{equation}
  \bd S_3 = \{\vec p \in S_3 : p_i = 0\ \textnormal{for some}\ i\}.
\end{equation}
Of special interest are the three points $\vec v_1 = (1,0,0)$, $\vec v_2 = (0,1,0)$ and $\vec v_3 = (0,0,1)$, corresponding to a state of dominance by one of the three grammars; these points are the \emph{vertices} of the simplex. In what follows, I will illustrate the behaviour of three-dimensional systems with the help of a barycentric triangular plot in which the vertices of the triangle correspond to the vertices of the simplex, the triangle's centroid corresponding to the mixed state $\vec p = (1/3, 1/3, 1/3)$ (Figure \ref{fig:simplex}).

\begin{figure}
  \centering
  \includegraphics[scale=0.85,clip,trim=0cm 1cm 0cm 1cm]{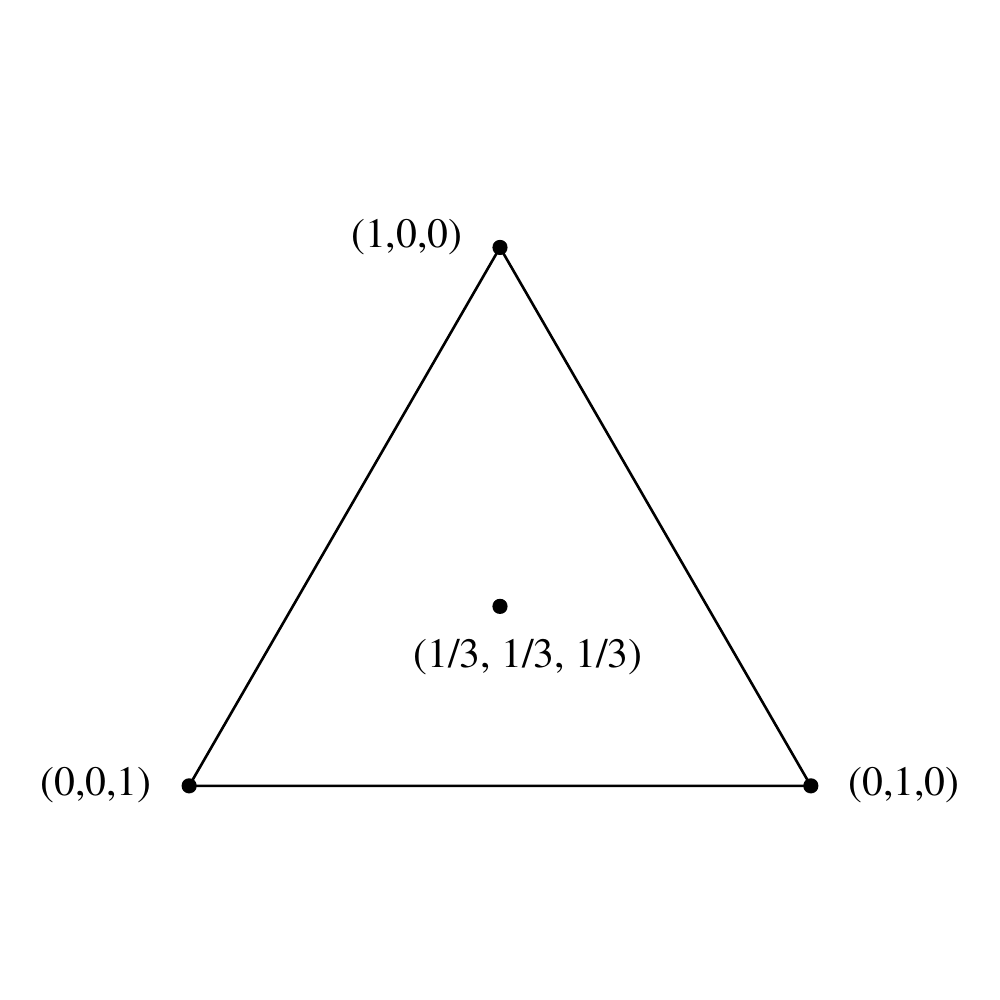}
  \caption{The state $\vec p = (p_1, p_2, p_3)$ of a $3$-grammar system is defined on the $3$-dimensional simplex $S_3$, which is best illustrated using a barycentric ternary plot. Shown here are the three vertices $\vec v_1 = (1,0,0)$, $\vec v_2 = (0,1,0)$ and $\vec v_3 = (0,0,1)$ as well as the barycentre $(1/3, 1/3, 1/3)$.}\label{fig:simplex}
\end{figure}

In the general case, the system \eqref{eq:3D-diff-eq} is too complicated to be solved analytically. In other words we do not have, for an arbitrary advantage matrix $\vec A$, a closed-form equation that would tell us the exact time evolution of the system from any given initial state. We can, however, arrive at an understanding of the system's dynamics by finding its rest points and studying their stability. A first result is that each of the three vertices $\vec v_i$ is a rest point and that no further rest points exist on the boundary $\bd S_3$, whenever $\vec A$ is proper:
\begin{theorem}
  The points $\vec v_1 = (1,0,0)$, $\vec v_2 = (0,1,0)$ and $\vec v_3 = (0,0,1)$ are rest points of \eqref{eq:3D-diff-eq} for any proper advantage matrix. No other point in $\bd S_3$ is a rest point.
\end{theorem}
\begin{proof}
  Using Theorem \ref{thm:zero-penalties}, inspection of \eqref{eq:3D-diff-eq} immediately shows that $\vec{v}_i' - \vec v_i = \vec 0$, i.e.~that each vertex $\vec v_i$ is a rest point.

  Now suppose that $\vec p = (p_1, p_2, 0)$ is a rest point. Then $p_3' - p_3 = 0$, which by \eqref{eq:3D-diff-eq} implies that $c_1c_2 = 0$, which implies that either $c_1 = 0$ or $c_2 = 0$. From Theorem \ref{thm:zero-penalties}, $p_1 = 1$ in the first case and $p_2 = 1$ in the second. Due to the symmetry of \eqref{eq:3D-diff-eq}, the same argument holds for states of the form $(p_1, 0, p_3)$ and $(0, p_2, p_3)$. Thus, if $\vec p \in \bd S_3$ is a rest point, it is necessarily a vertex.
\end{proof}
\noindent If an interior rest point exists, it satisfies a stability condition:
\begin{theorem}\label{thm:rest-point-condition}
  Let $\vec p = (p_1, p_2, p_3) \in \intr S_3$. Then $\vec p$ is a rest point if and only if $c_1p_1 = c_2p_2 = c_3p_3$.
\end{theorem}
\begin{proof}
  Since $\vec p\in \intr S_3$, Theorem \ref{thm:zero-penalties} implies that $c_i > 0$ for all $i$. Division by the $c_i$ is then possible, and \eqref{eq:nD-diff-eq} reduces, with algebra, to
  \[
    p_i' = \frac{c_i^{-1}}{\sum_j c_j^{-1}} = \frac{1}{c_i\sum_j c_j^{-1}}.
  \]
Now $p_i' - p_i = 0$ if and only if
  \[
    c_ip_i = \frac{1}{\sum_j c_j^{-1}}.
  \]
This holds for all $i$ and the right hand side is independent of $i$. Hence, the previous is equivalent to $c_1p_1 = c_2p_2 = c_3p_3$.
\end{proof}

Apart from these simple observations, it is difficult to obtain further results concerning the behaviour of \eqref{eq:3D-diff-eq} in the general case. I will next turn to a consideration of a number of special cases which are considerably easier to analyse, in increasing order of complexity, so as to arrive at a general picture of the diachronic behaviour of multiple-grammar systems based on LRP learning.

\section{Babelian systems}\label{sec:babelian}

Arguably the simplest case occurs when all of the pairwise advantages $a_{ij}$ are equal -- in this case, no single grammar has a net benefit over the rest. Formally, I will say that a system is \emph{Babelian} if its advantage matrix satisfies the following: there is an $a > 0$ such that $a_{ij} = a$ for all $i,j$ with $i \neq j$. In three dimensions, this amounts to matrices of the form
\begin{equation}
  \vec A = \begin{bmatrix}
    0 & a & a \\
    a & 0 & a \\
    a & a & 0
  \end{bmatrix}.
\end{equation}
Notice that such matrices satisfy the cyclical balance criterion \eqref{eq:cyclical-balance-criterion} and are thus valid advantage matrices.

Any Babelian $3$-grammar system turns out to have one interior rest point, namely the maximum entropy state $(1/3, 1/3, 1/3)$:
\begin{theorem}\label{thm:babelian-interior-rest-point}
  For any Babelian $3$-grammar system, the state $(1/3, 1/3, 1/3)$ is the only interior rest point.
\end{theorem}
\begin{proof}
  That $(1/3, 1/3, 1/3)$ is a rest point would be easy to establish using Theorem \ref{thm:rest-point-condition}. To prove the stronger result that it is the only interior rest point of a Babelian system, let us look at the difference equation \eqref{eq:3D-diff-eq} directly. In the interior $\intr S_3$, one has (cf.~proof of Theorem \ref{thm:rest-point-condition})
  \[
    \begin{split}
      p_i' - p_i &= \frac{c_i^{-1}}{\sum_j c_j^{-1}} - p_i \\
      &= \frac{(\sum_k a_{ik}p_k)^{-1}}{\sum_j(\sum_k a_{jk} p_k)^{-1}} - p_i \\
      &= \frac{(\sum_k a p_k)^{-1}}{\sum_j (\sum_k a p_k)^{-1}} - p_i \\
      &= \frac{a^{-1} (\sum_k p_k)^{-1}}{a^{-1} \sum_j (\sum_k p_k)^{-1}} - p_i 
    \end{split}
  \]
  for a Babelian system. But $\sum_k p_k = 1$, so the above is equivalent to
  \[
      p_i' - p_i
      = \frac{a^{-1}}{3a^{-1}} - p_i 
      = \frac{1}{3} - p_i
  \]
in three dimensions. Hence $p_i' - p_i = 0$ if and only if $p_i = 1/3$, and consequently $(1/3, 1/3, 1/3)$ is the only interior rest point.
\end{proof}

Thus any Babelian three-grammar system has four rest points: the three vertices, corresponding to total dominance by one of the three grammars, and the maximum entropy state in which each grammar has equal representation. It remains to figure out the stability of these rest points. In general, stability analysis hinges on studying how the state of the dynamical system under consideration changes in the immediate vicinity of the rest point in question -- whether nearby points in the system's state space are attracted to the rest point or repelled by it (cf.~our discussion of the three pendula in Section \ref{sec:introduction}). Mathematically, we need to study the partial derivatives of the system's evolution equations when evaluated at the rest point. For a three-dimensional system, the \emph{Jacobian matrix} is defined as the matrix of partial derivatives
\begin{equation}
    \renewcommand{\arraystretch}{1.5}
  \vec J = \begin{bmatrix}
    \frac{\partial f_1}{\partial p_1} & \frac{\partial f_1}{\partial p_2} & \frac{\partial f_1}{\partial p_3} \\
    \frac{\partial f_2}{\partial p_1} & \frac{\partial f_2}{\partial p_2} & \frac{\partial f_2}{\partial p_3} \\
    \frac{\partial f_3}{\partial p_1} & \frac{\partial f_3}{\partial p_2} & \frac{\partial f_3}{\partial p_3}
  \end{bmatrix}
\end{equation}
where the functions $f_i$ are as in \eqref{eq:gen-diff-eq}. When the partial derivatives $\partial f_i/\partial p_j$ are evaluated at a rest point $\vec p = (p_1, p_2, p_3)$, the Jacobian reduces to a matrix of real numbers; denote this by $\vec J (\vec p)$. It can then be shown that, for a discrete-time system, (1) if the modulus of each eigenvalue of $\vec J(\vec p)$ is strictly less than $1$, the rest point $\vec p$ is asymptotically stable, and (2) if the modulus of at least one eigenvalue is strictly greater than $1$, $\vec p$ is unstable \citep[70--71]{Dra1992}. While this method is foolproof in the sense that it is purely a matter of mechanical calculation, computing the eigenvalues is in most cases extremely tedious and is best left to a computer. In what follows, I shall consequently only report the end results of these computations, suppressing the gritty details.

Applying the Jacobian method on \eqref{eq:3D-diff-eq} gives us our main result on the stability of Babelian systems.
\begin{theorem}
  In a three-dimensional Babelian system, the interior rest point $(1/3, 1/3, 1/3)$ is asymptotically stable. The vertex rest points $(1,0,0)$, $(0,1,0)$ and $(0,0,1)$ are all unstable.
\end{theorem}
\begin{proof}
  The Jacobian has eigenvalues $0$ and $2$ at each of the three vertices, and eigenvalues $0$ and $1/2$ at the interior rest point $(1/3, 1/3, 1/3)$.
\end{proof}
\noindent Thus, as expected, the natural tendency in a Babelian system is away from dominance and towards the maximally mixed state $(1/3, 1/3, 1/3)$ in which each grammar is used with probability $1/3$ (Figures \ref{fig:babelian-ternary}--\ref{fig:babelian-timedim}). This shows that three-grammar Babelian systems have ``built-in'' stable variation, in stark contrast to the two-grammar case (Section \ref{sec:2D}).

\begin{figure}
  \centering
  \includegraphics[scale=0.85,clip,trim=0cm 1cm 0cm 1cm]{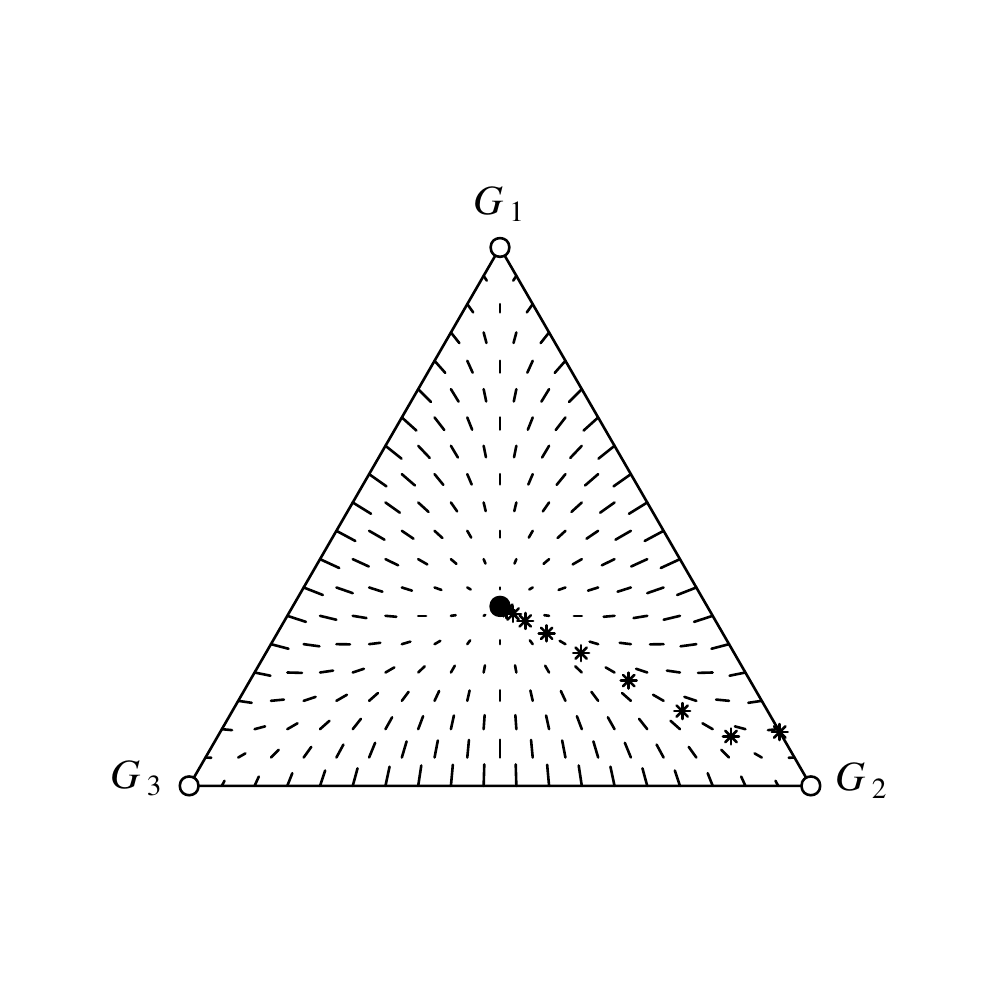}
  \caption{Phase space plot of Babelian $3$-grammar systems. The three unstable vertex rest points are shown as open circles and the stable interior rest point as a filled circle, as is customary; the line segments give the magnitude and direction of change at various points in the state space. The series of asterisks illustrates one diachronic (inter-generational) trajectory from the initial state $\vec p = (0.1, 0.9, 0.0)$; see Figure \ref{fig:babelian-timedim} for a conventional representation of this trajectory in the time dimension.}\label{fig:babelian-ternary}
\end{figure}

\begin{figure}
  \centering
  \includegraphics[scale=0.85]{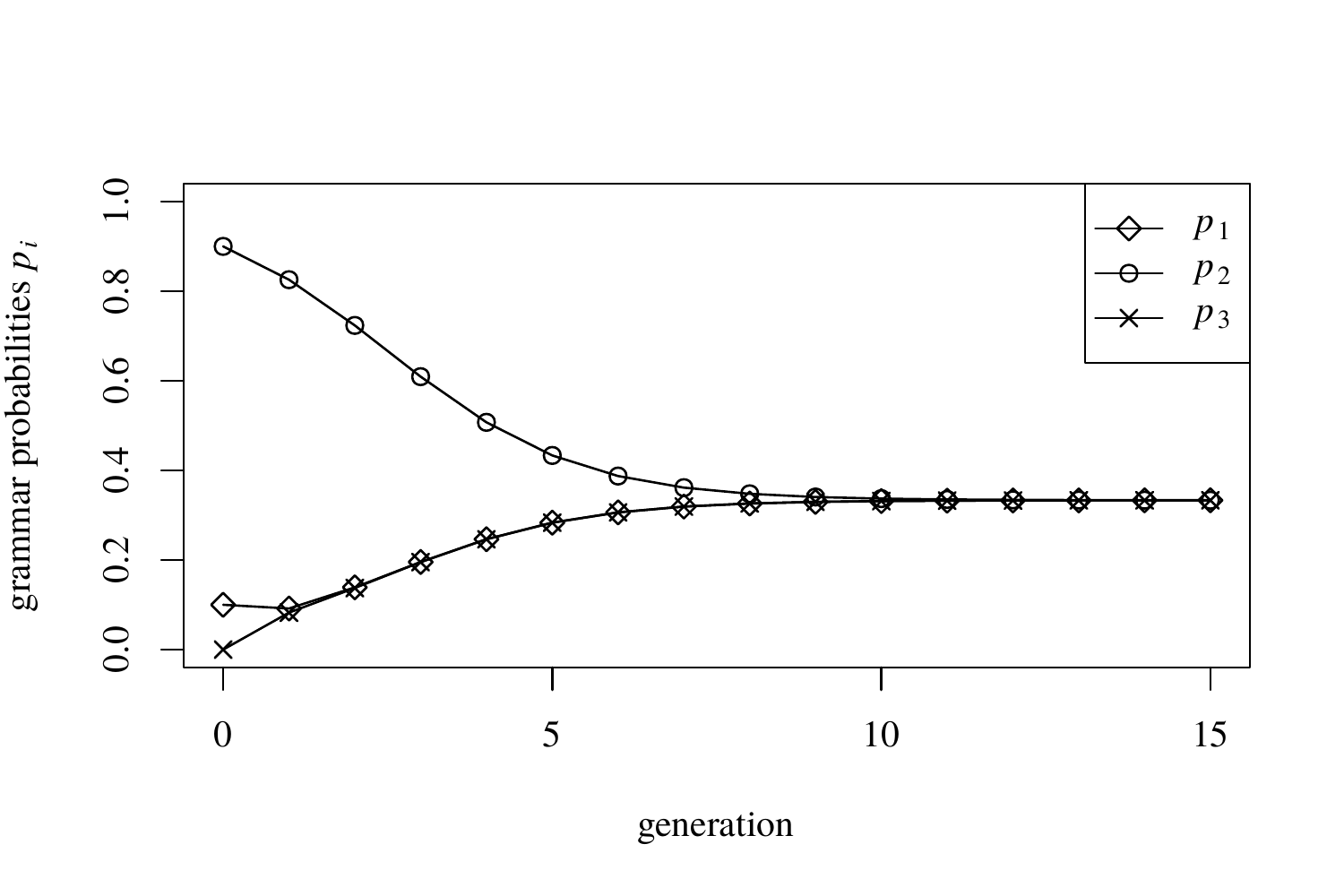}
  \caption{The trajectory from Figure \ref{fig:babelian-ternary} shown in the time dimension.}\label{fig:babelian-timedim}
\end{figure}

\section{Symmetric systems}\label{sec:symmetric}

The above analysis illustrates the procedure of sketching the qualitative behaviour of a dynamical system by way of analysing the system's rest points and their stability, when the equations governing the system's evolution cannot be solved. It also shows that true stable variation is a feature of at least some formal systems based on LRP learning. Babelian systems, of course, are far too trivial to be of any serious linguistic interest, and it remains to show that stable variation may occur in other, more realistic multiple-grammar settings.

A straightforward way of generalizing from Babelian systems is to allow some of the grammars to have unequal advantages but to maintain a symmetry condition: $a_{ij} = a_{ji}$ for all $i,j$. In three dimensions, such \emph{symmetric} systems are thus described by advantage matrices of the form
\begin{equation}\label{eq:symmetric-matrix}
  \vec A = \begin{bmatrix}
    0 & a_{12} & a_{13} \\
    a_{12} & 0 & a_{23} \\
    a_{13} & a_{23} & 0
  \end{bmatrix}
  = \begin{bmatrix}
    0 & a & b \\
    a & 0 & c \\
    b & c & 0
  \end{bmatrix}
\end{equation}
where I write $a = a_{12}$, $b = a_{13}$ and $c = a_{23}$ for convenience. Again, it is clear that these matrices satisfy the cyclical balance criterion \eqref{eq:cyclical-balance-criterion} and thus are well-defined.

Setting $p_i' - p_i = 0$ in \eqref{eq:3D-diff-eq} and solving for $p_i$ (in a manner analogous to that in the proof of Theorem \ref{thm:babelian-interior-rest-point} above) reveals that in a symmetric three-grammar system, a rest point exists at
\begin{equation}
  \vec p = \left(\frac{c}{a+b+c}, \frac{b}{a+b+c}, \frac{a}{a+b+c}\right).
\end{equation}
Continuing to assume proper advantage matrices, in other words that $a,b,c > 0$, it follows that this rest point is always contained in the interior $\intr S_3$. It is also the only solution of $p_i' - p_i = 0$ in the interior and hence the only interior rest point of a symmetric system. Furthermore, stability analysis finds that the Jacobian, when evaluated at this rest point, has eigenvalues $0 < 1$ and $1/2 < 1$; hence, the interior rest point is always asymptotically stable. For each of the vertex rest points $\vec v_1$, $\vec v_2$ and $\vec v_3$, the eigenvalues are $0 < 1$ and $2 > 1$. Thus:
\begin{theorem}\label{thm:symmetric-stable}
  Any proper, symmetric three-grammar system \eqref{eq:symmetric-matrix} has exactly one interior rest point at
  \[
  \vec p = \left(\frac{c}{a+b+c}, \frac{b}{a+b+c}, \frac{a}{a+b+c}\right).
  \]
This interior rest point is asymptotically stable, while the vertex rest points are all unstable.\qed
\end{theorem}
\noindent Crucially, the result holds for any values of $a,b,c > 0$. We then conclude:
\begin{corollary}
  Any proper, symmetric system of three grammars tends to a state of stable variation.\qed
\end{corollary}
\noindent Figure \ref{fig:symmetric-ternary} illustrates for a particular choice of the parameters $a$, $b$ and $c$.

\begin{figure}
  \centering
  \includegraphics[scale=0.85,clip,trim=0cm 1cm 0cm 1cm]{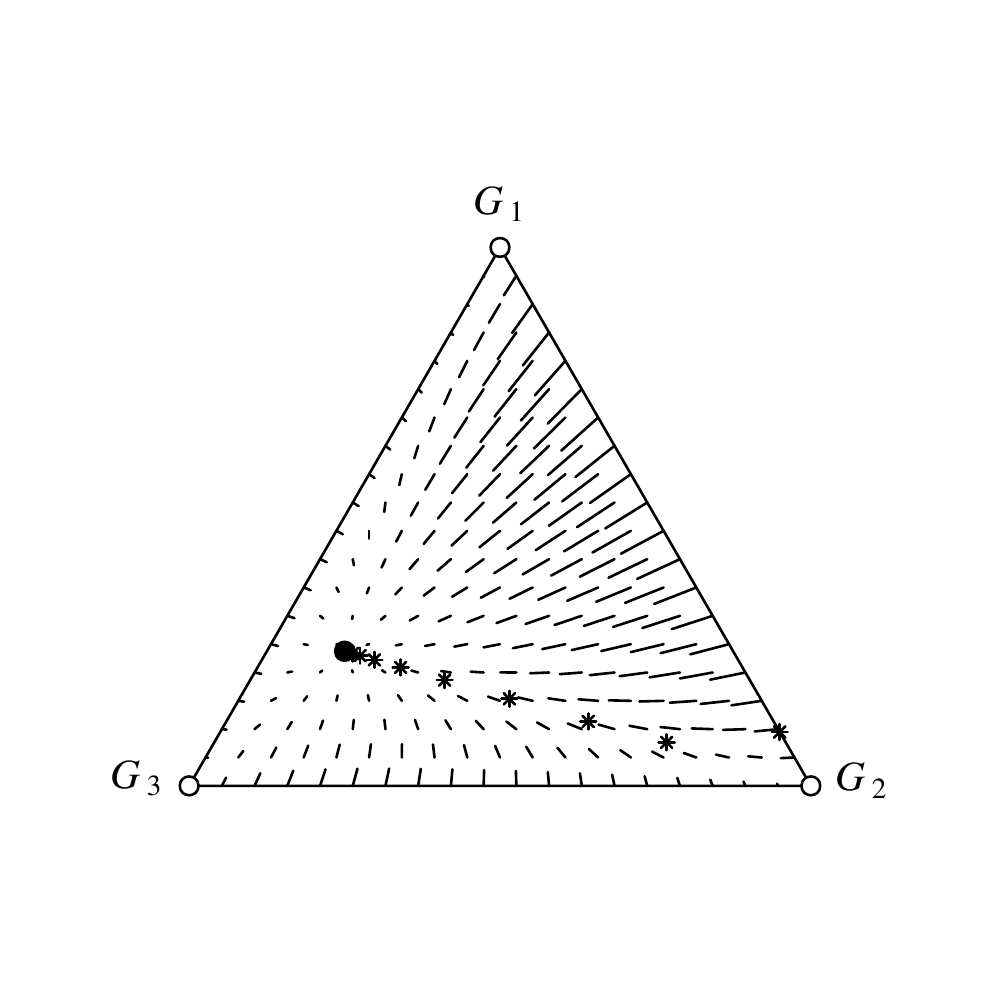}
  \caption{Phase space plot of a symmetric $3$-grammar system with $a = 0.05$, $b = 0.01$ and $c = 0.02$. Each trajectory not starting at a vertex point tends towards the stable interior rest point at $(c/D, b/D, a/D)$ with $D = a+b+c$.}\label{fig:symmetric-ternary}
\end{figure}

\section{Quasi-Babelian systems}\label{sec:quasi-babelian}

Another way of generalizing from the Babelian special case is to explore a more comprehensive class of systems in which some one grammar has either a larger or a smaller advantage than any of its competitors, the latter sharing the same amount of advantage amongst themselves. Formally, I will call a system \emph{quasi-Babelian} if constants $a,b > 0$ exist such that (1) for some unique $i$, $a_{ji} = b$ for all $j\neq i$, and (2) $a_{kj} = a$ for all $j\neq i$, for all $k\neq j$. By a relabelling of grammars, we may always take $G_1$ to correspond to the grammar having the unique advantage $b$, and I will refer to this as the \emph{canonical} quasi-Babelian case. In three dimensions, a canonical quasi-Babelian advantage matrix, then, is of the form
\begin{equation}
  \vec A = \begin{bmatrix}
    0 & a & a \\
    b & 0 & a \\
    b & a & 0
  \end{bmatrix}.
\end{equation}
Again, it can be checked that the cyclical balance criterion \eqref{eq:cyclical-balance-criterion} is satisfied.

The advantage matrix now has just two independent parameters, $a$ and $b$, and consequently algebraic manipulation of the equations \eqref{eq:3D-diff-eq} becomes easy. Setting $p_i' - p_i = 0$ and solving for $p_i$ reveals that with a canonical quasi-Babelian advantage matrix, \eqref{eq:3D-diff-eq} has either three or four rest points in the simplex $S_3$. In addition to the vertices $\vec v_1$, $\vec v_2$ and $\vec v_3$, a fourth solution exists in the interior at the point
\begin{equation}
  \vec p^* = \left( \frac{1}{5-2\rho}, \frac{2-\rho}{5-2\rho}, \frac{2-\rho}{5-2\rho} \right)
\end{equation}
whenever $0 < \rho < 2$, where $\rho = b/a$ gives the ratio of the two advantage parameters. At $\rho = 2$, this solution coalesces with the vertex $\vec v_1$.

This rest point $\vec p^*$ entails a sort of behaviour which is entirely unattested in Babelian and symmetric systems: a bifurcation. For small values of the ratio $\rho = b/a$ -- that is, for values of $b$ which are small in comparison to $a$ -- the interior rest point $\vec p^*$ exists. As $\rho$ is increased, this rest point moves towards the vertex $\vec v_1$ and coincides with the latter at the critical value $\rho = \rho_c = 2$ of the bifurcation parameter $\rho$. For ratios $\rho \geq 2$, the system consequently only has the three vertex rest points. The following theorem establishes the stability of these rest points in response to the bifurcation; Figures \ref{fig:quasi-sequence}--\ref{fig:quasi-bifurcation} illustrate.
\begin{theorem}
  Assume a canonical quasi-Babelian $3$-grammar system with advantage ratio $\rho = b/a$. Then
  \begin{enumerate}
    \item the vertex rest points $\vec v_2 = (0,1,0)$ and $\vec v_3 = (0,0,1)$ are always unstable;
    \item the vertex rest point $\vec v_1 = (1,0,0)$ is asymptotically stable if $\rho \geq 2$ and unstable if $0 < \rho < 2$;
    \item the interior fixed point $\vec p^* = \left( \frac{1}{5-2\rho}, \frac{2-\rho}{5-2\rho}, \frac{2-\rho}{5-2\rho} \right)$ is asymptotically stable whenever it exists, i.e.~when $0 < \rho < 2$.
  \end{enumerate}
\end{theorem}
\begin{proof}
  For the two vertices $\vec v_2$ and $\vec v_3$, the eigenvalues of the Jacobian are $0$ and $1 + \rho > 1$. Hence, these points are unstable.

  At the vertex $\vec v_1$, the Jacobian has eigenvalues $0$ and $2a/b$. Hence, this rest point is asymptotically stable if $2a/b < 1$, i.e.~if $b/a = \rho > 2$, and unstable if $b/a = \rho < 2$.

  At the interior rest point $\vec p^*$, the Jacobian has eigenvalues $0 < 1$, $1/2\rho$ and $1 - 1/2\rho < 1$. Thus, the interior rest point is asymptotically stable whenever $1/2\rho < 1$, i.e.~when $\rho < 2$.
\end{proof}

\begin{figure}
  \centering
  \includegraphics[width=\textwidth]{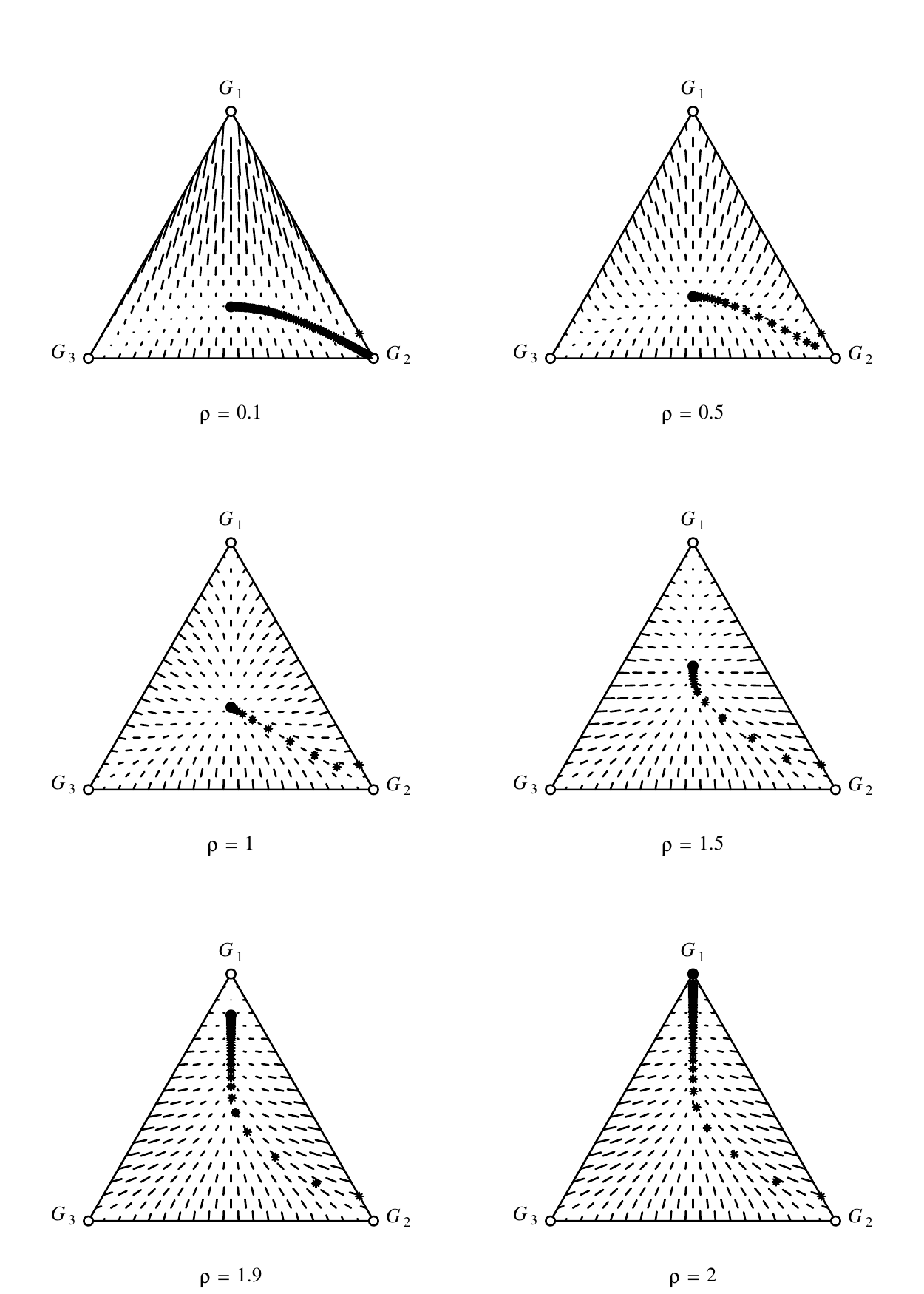}
  \caption{Phase space plots of the canonical quasi-Babelian $3$-grammar system for various advantage ratios $\rho = b/a$; $\rho =1$ corresponds to the strictly Babelian special case. At $\rho = 2$ a bifurcation occurs in which the interior rest point joins the vertex $\vec v_1$, reversing the latter's stability.}\label{fig:quasi-sequence}
\end{figure}

\begin{figure}
  \centering
  \includegraphics[scale=0.85]{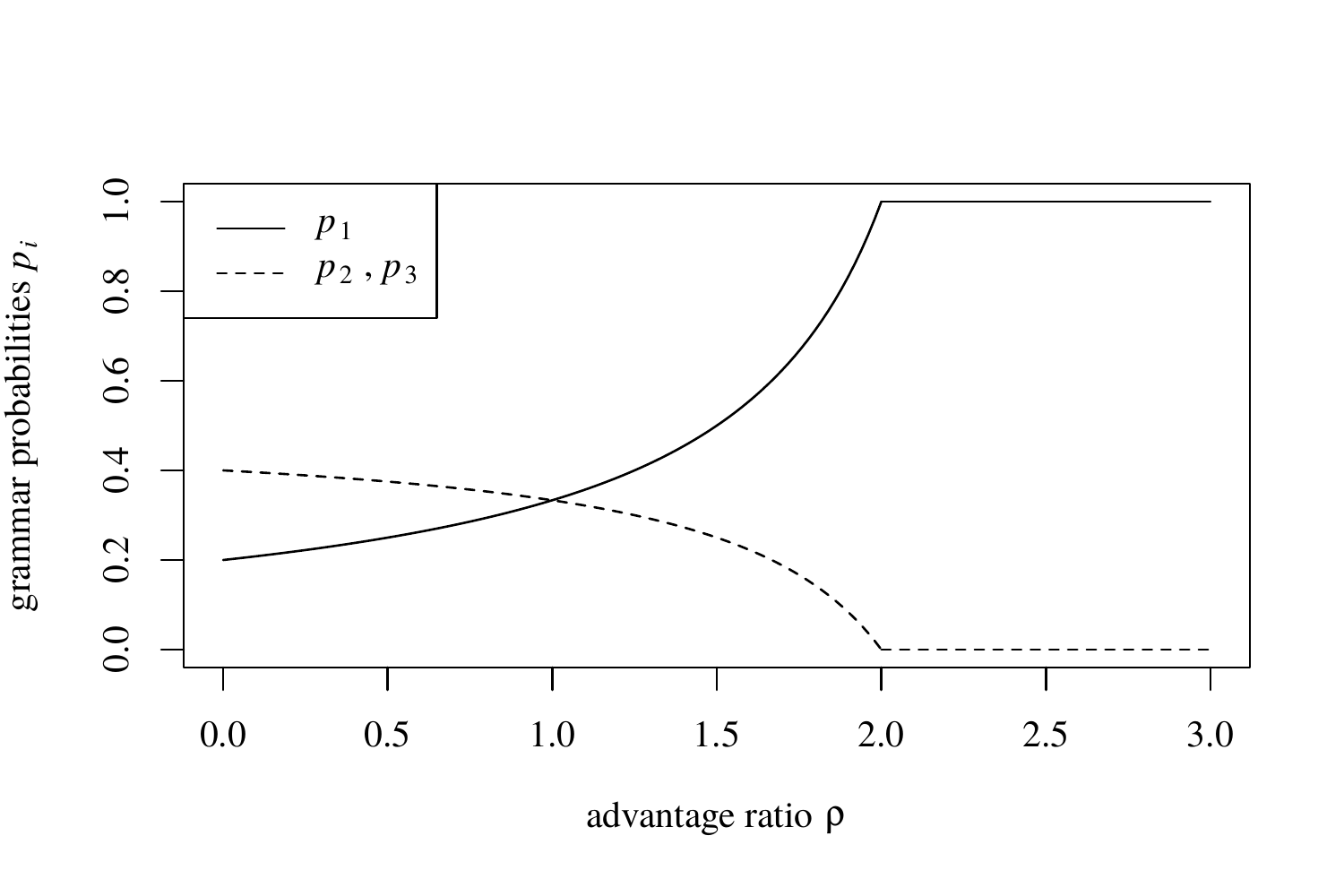}
  \caption{Orbit diagram of quasi-Babelian $3$-grammar systems, illustrating the stable limiting state of the system when started from any non-vertex state. The solid curve gives the value of $p_1$ at the stable rest point, while the dashed curve gives the value of $p_2 = p_3$.}\label{fig:quasi-bifurcation}
\end{figure}

\section{Naive learning}\label{sec:naive-learning}

Above, I have explored a generalization of the $2$-grammar Variational Learner. This generalization has shown that stable variation is an intrinsic feature of many multiple-grammar systems based on LRP learning. The specific systems studied and their interrelationships are summarized in Figure \ref{fig:systems}; future work will need to explore systems that lie outside these classes of systems.

\begin{figure}
  \centering
  \includegraphics[scale=0.65]{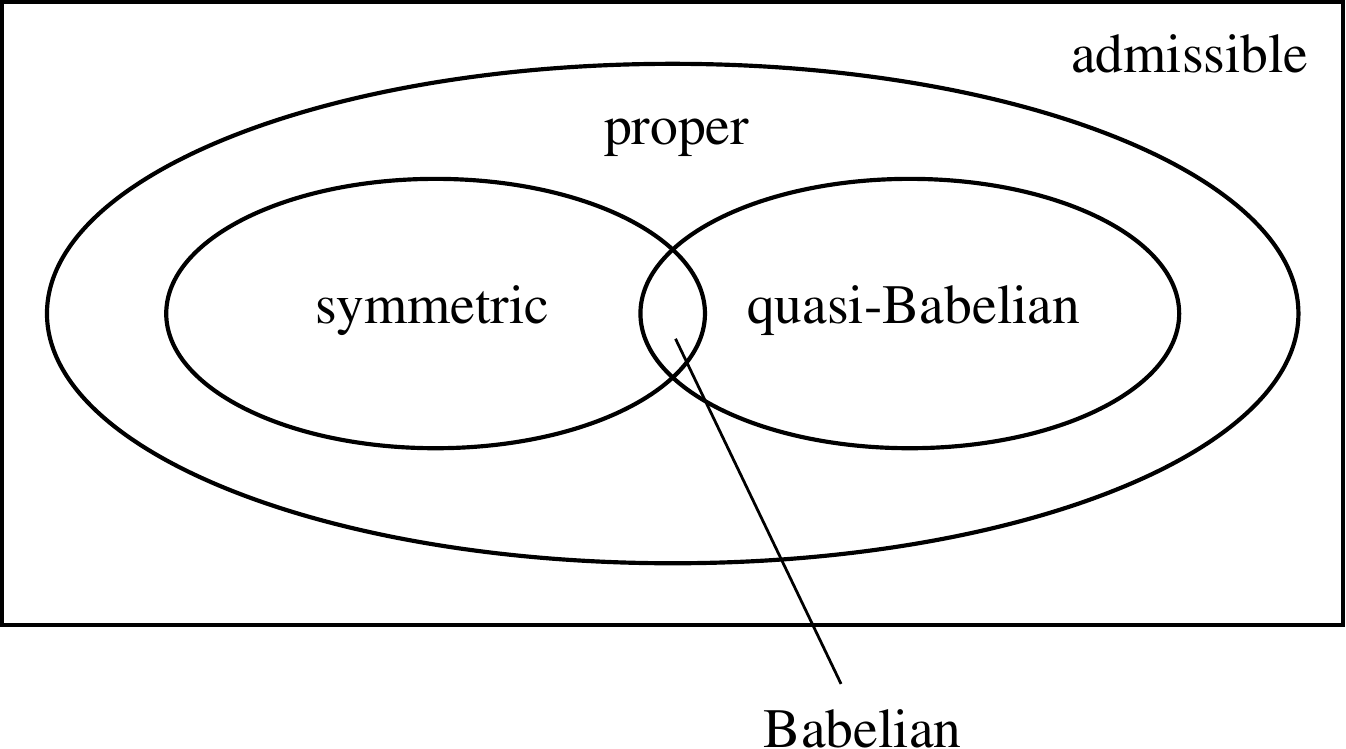}
  \caption{Set relations among the $3$-grammar systems studied in this paper, in the universe of all admissible systems (all $3\times 3$ advantage matrices satisfying the cyclical balance criterion): all Babelian systems are both symmetric and quasi-Babelian, and all symmetric and quasi-Babelian systems are proper.}\label{fig:systems}
\end{figure}

Crucially, the preceding analysis relies on the straightforward generalization of LRP learning for $n$ options given in Algorithm \ref{alg:LRP-nD}. From a psycholinguistic point of view, this way of treating the learner implies, for better or worse, that the learner must keep track of $n$ independent probabilities. Considering that even a few dozens of (binary) grammatical parameters result in an astronomical search space for the learner, the straightforward extension of the LRP algorithm may be argued to be unrealistic on psychological grounds.\footnote{The issue is in fact convoluted: on the one hand, the number of grammatical parameters is not known with any certainty (for one recent estimate, see \citealt[1687]{LonGua2009}, who suggest $63$ parameters in the DP domain and note that in general ``UG parameters number at least in the hundreds''), and on the other hand, the human brain \emph{is} capable of storing astronomical quantities of information \citep{BarEtal2015}. I set the issue aside here -- for present purposes, what matters is that stable variation is attested both in the straightforward $n$-grammar generalization of LRP learning and in the parametrically constrained Naive Learner, as we will presently see.}

An alternative, explored to some extent in \citet{Yan2002}, is to have the learner operate in a parametrically constrained space. That is to say, instead of operating on $n$ \emph{grammar} probabilities $\pi_1, \dots , \pi_n$, suppose the learner operates on $N$ \emph{parameter} probabilities $\xi_1, \dots ,\xi_N$, where $n = 2^N$ and $\xi_i$ gives the probability of the $i$th binary parameter being set on. To recover the grammar probabilities, it suffices to multiply the relevant parameter probabilities:
\begin{equation}
  P(G_{\sigma (1)\sigma (2)\dots \sigma (N)}) = \prod_{i=1}^N \xi_i^{\sigma (i)} (1-\xi_i)^{1-\sigma (i)}
\end{equation}
is the probability of the grammar $G_{\sigma (1)\sigma (2)\dots \sigma (N)}$ being selected, with $\sigma (i) = 1$ if the $i$th parameter is to be set on and $\sigma (i) = 0$ if the $i$th parameter is to be set off for this particular grammar.

Since what gets rewarded or punished is the selection of entire grammars and not the selection of individual parameter values, the learner now faces the problem of not knowing which parameter setting(s) to blame in case of parsing failure \citep{Yan2002}. One way of attempting to overcome this problem is the following naive learning algorithm.
\begin{algo}[Naive Parameter Learner (NPL); \citealp{Yan2002}]
  \leavevmode\vspace{-\baselineskip}
  \begin{enumerate}
    \item Set $\xi_i = 0.5$ for all $i$ initially.
    \item Pick grammar by setting $i$th parameter on with probability $\xi_i$.
    \item Receive input sentence $x$.
    \item If grammar parses $x$:
      \begin{enumerate}[label=\alph*.]
        \item If $i$th parameter was on, increase the value of $\xi_i$ by replacing $\xi_i$ with $\xi_i + \gamma (1-\xi_i)$, where $\gamma$ is a learning rate.
        \item Else decrease the value of $\xi_i$ by replacing it with $(1-\gamma)\xi_i$.
      \end{enumerate}
    \item If grammar does not parse $x$:
      \begin{enumerate}[label=\alph*.]
        \item If $i$th parameter was on, decrease the value of $\xi_i$ by replacing it with $(1-\gamma)\xi_i$.
        \item Else increase the value of $\xi_i$ by replacing it with $\xi_i + \gamma(1-\xi_i)$.
      \end{enumerate}
    \item Repeat steps 2--5 for $T$ input tokens.
  \end{enumerate}
\end{algo}

Having learners operate in a parametrically constrained space and employing a learning algorithm such as NPL complicates the study of the diachronic behaviour of such a system, since analogues of the learning-theoretic limiting approximations \eqref{eq:2D-LRP-approximation} and \eqref{eq:nD-LRP-approximation} are not available. It is, however, possible to study special cases with the help of computer simulations. In what follows, I will explore one such simple special case and show that stable variation is, again, a feature of at least some systems based on Naive Parameter Learning in a parametric space.

For this, suppose for simplicity that Universal Grammar (UG) provides just two elements, a ``noun'' N and a ``determiner'' D, and two parameters:
\begin{enumerate}
  \item whether determiner can be null (on setting) or has to be overt (off setting)
  \item whether grammar is head-final (on setting) or head-initial (off setting)
\end{enumerate}
The four grammars then parse, and fail to parse, strings as follows:
\begin{center}
  \begin{tabular}{l|ll}
    & parses & fails to parse \\
    \hline
    $G_{11}$ & N, DN & ND\\
    $G_{10}$ & N, ND & DN\\
    $G_{01}$ & DN & N, ND\\
    $G_{00}$ & ND & N, DN
  \end{tabular}
\end{center}
Assuming true optionality, i.e.~that grammars $G_{11}$ and $G_{10}$ generate the two types of sentence with probability $0.5$, it is easy to work out the probability of each possible input string the learner may encounter:
\begin{equation}
  \begin{split}
      P(\textnormal{N}) &= 0.5 P(G_{11}) + 0.5 P(G_{10}) \\
      &= 0.5 x_1x_2 + 0.5 x_1(1 - x_2) \\
      &= 0.5 x_1 \\
      & \\
      P(\textnormal{DN}) &= 0.5 P(G_{11}) + P(G_{01}) \\
      &= 0.5 x_1x_2 + (1-x_1)x_2 \\
      &= x_2(1 - 0.5 x_1) \\
      & \\
      P(\textnormal{ND}) &= 0.5 P(G_{10}) + P(G_{00}) \\
      &= 0.5 x_1(1-x_2) + (1-x_1)(1-x_2) \\
      &= (1-x_2)(1 - 0.5x_1)
  \end{split}
\end{equation}
Here, $x_1$ and $x_2$ are the population-level parameter probabilities (corresponding to $p_i$ in the LRP formulation). The penalty probabilities of the four grammars are then found to be
\begin{equation}\label{eq:NPL-penalties}
  \left\{
    \begin{aligned}
      c(G_{11}) &= (1-x_2)(1 - 0.5x_1) \\
      c(G_{10}) &= x_2(1 - 0.5x_1) \\
      c(G_{01}) &= 0.5x_1 + (1-x_2)(1 - 0.5x_1) \\
      c(G_{00}) &= 0.5x_1 + x_2(1 - 0.5x_1)
    \end{aligned}
  \right.
\end{equation}
Substituting $x_1 = x_2 = 1$ in \eqref{eq:NPL-penalties} yields
\begin{equation}
  \left\{
    \begin{aligned}
      c(G_{11}) &= 0 \\
      c(G_{10}) &= 0.5 \\
      c(G_{01}) &= 0.5 \\
      c(G_{00}) &= 1
    \end{aligned}
  \right.
\end{equation}
which shows that if $G_{11}$ is the unique target grammar, then the NPL algorithm will eventually arrive at the right parameter probabilities $\xi_1 = 1$ and $\xi_2 = 1$, as long as the learner has enough time to tweak the probabilities (Figure \ref{fig:NPL-learner}). Performing the requisite substitutions shows that the same holds for the remaining three grammars $G_{10}$, $G_{01}$ and $G_{00}$, as well.

\begin{figure}
  \centering
  \includegraphics[scale=0.85]{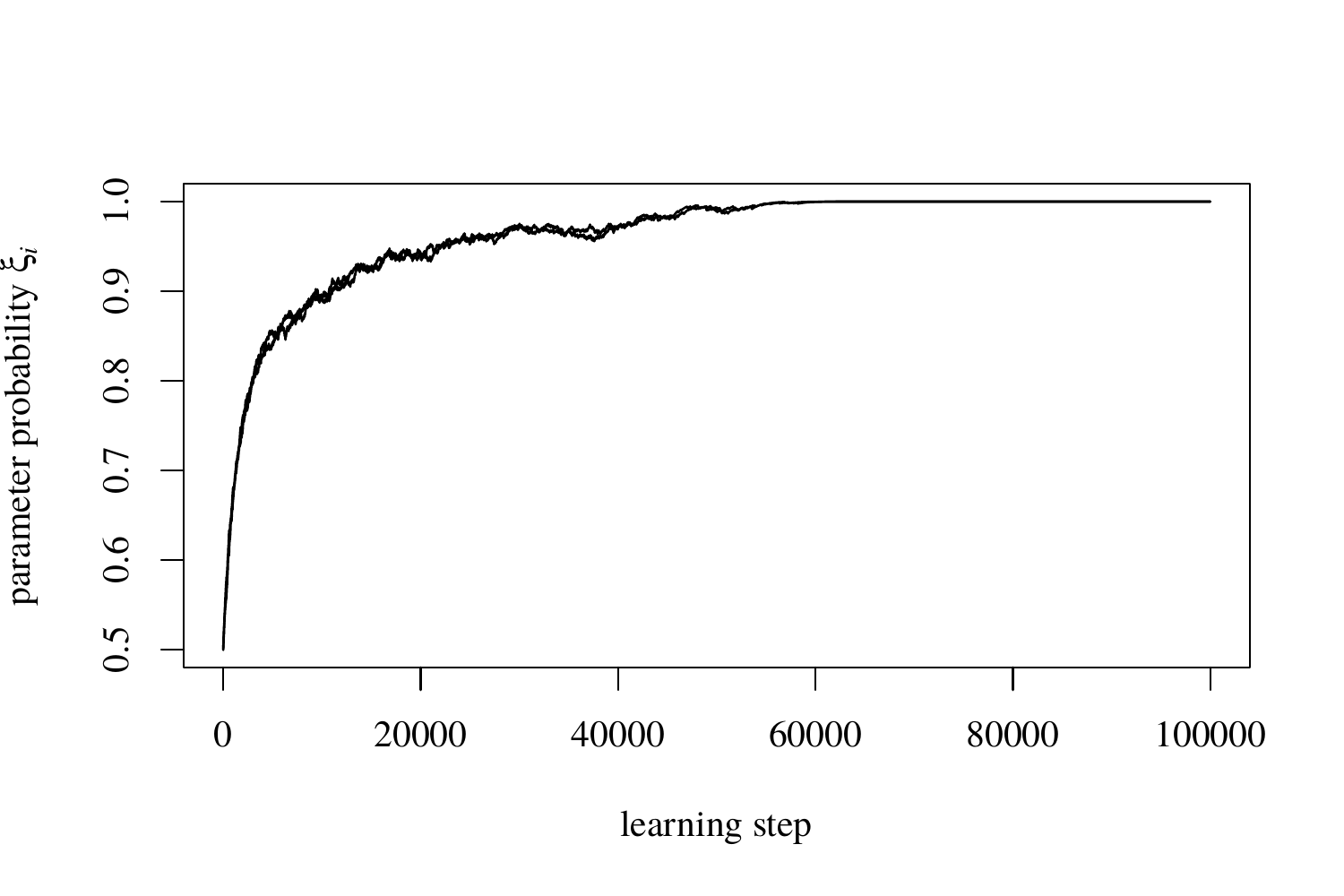}
  \caption{A two-parameter Naive Parameter Learner at the vertex $\vec x = (1,1)$ ($G_{11}$ is the unique target grammar); values of $\xi_1$ and $\xi_2$ from one computer simulation. Both of the learner's parameter probabilities $\xi_1$ and $\xi_2$ tend to $1$.}\label{fig:NPL-learner}
\end{figure}

The four vertices, at which one of the four grammars has total use, are thus found to be rest points for the above toy system. What about their stability? To explore this question, we need to set the learner in a mixed environment (at a state in the interior $\intr S_4$ of the four-simplex of \emph{grammar} probabilities). Figure \ref{fig:NPL-mixed-learner} shows the behaviour of the learner in the mixed environment $(x_1, x_2) = (0.99, 0.99)$, corresponding to $P(G_{11}) = 0.99^2 = 0.9801$ use of the grammar $G_{11}$. Convergence to the vertex no longer occurs, and the diachronic implications of this become manifest when we set up a sequence of such learners, the output of one generation again feeding as input to the following generation: when started from a mixed state, the system fails to converge to the vertex rest point at which $G_{11}$ has dominance, and instead appears to be attracted to an interior rest point, that is to say, towards a state of stable variation (Figure \ref{fig:NPL-mixed-diachrony}).

\begin{figure}
  \centering
  \includegraphics[scale=0.85]{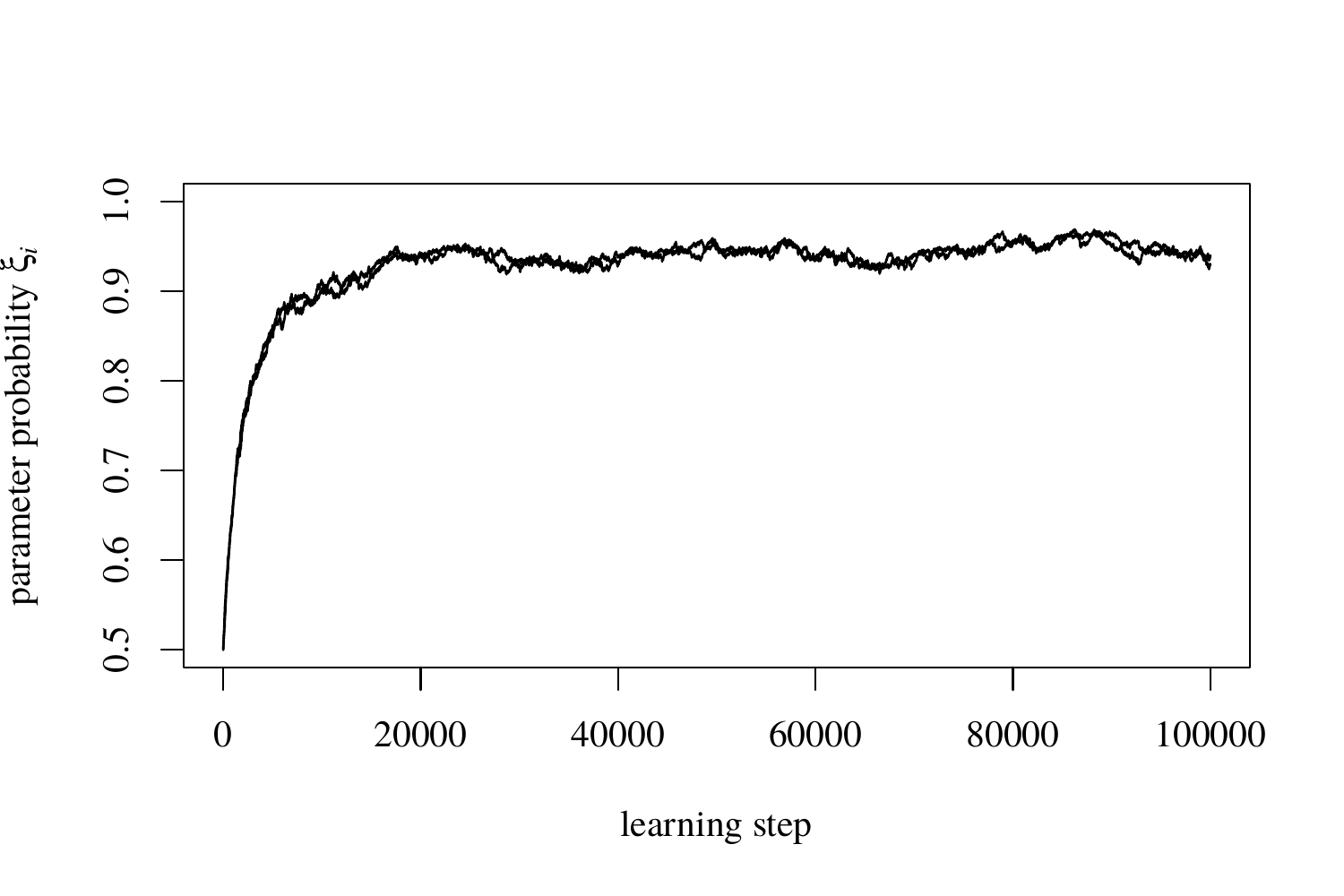}
  \caption{A two-parameter Naive Parameter Learner at the interior point $\vec x = (0.99, 0.99)$.}\label{fig:NPL-mixed-learner}
\end{figure}

\begin{figure}
  \centering
  \includegraphics[scale=0.85]{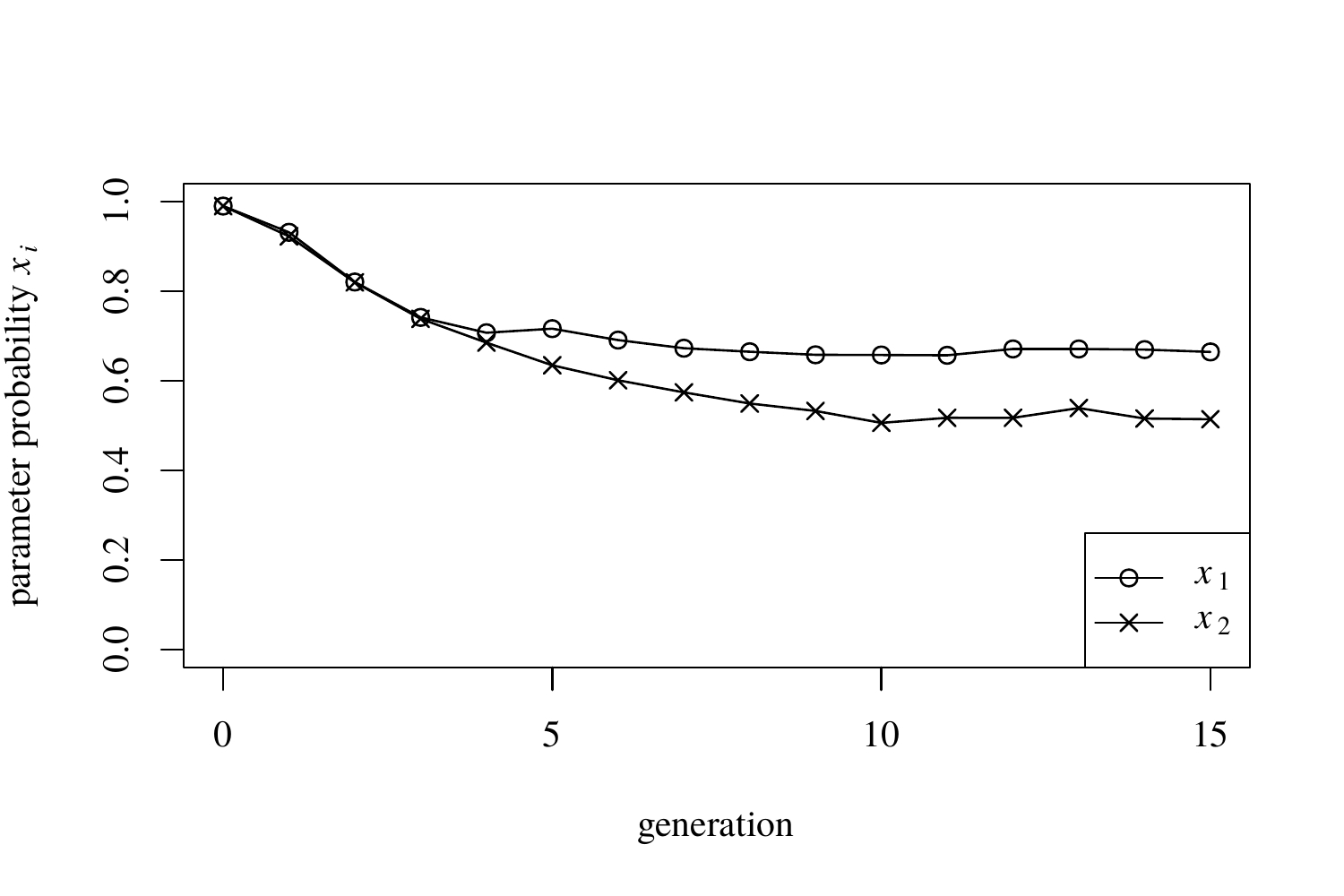}
  \caption{Diachrony for a sequence of Naive Parameter Learners from the initial state $\vec x = (0.99, 0.99)$. Convergence to the vertex $\vec x = (1,1)$ does not occur, suggesting that this vertex is an unstable rest point.}\label{fig:NPL-mixed-diachrony}
\end{figure}

\section{Conclusions and conjectures}

In this paper, I have shown that diachronically stable variation arises in many kinds of settings of grammar competition, as long as more than two grammars are represented in the learner's environment. In addition to a systematic study of the $n$-grammar LRP learning algorithm in Sections \ref{sec:2D}--\ref{sec:quasi-babelian}, the preliminary exploration of a toy parametric UG in Section \ref{sec:naive-learning} points to the conclusion that stable variation occurs in the parametrically constrained Naive Parameter Learner as well.

The results of this paper invite experimentalists to look for evidence of stable variation in a specific kind of situation -- complex language contact. Indeed, given Yang's \citeyearpar{Yan2000} Fundamental Theorem, more than two grammars \emph{must} be present in the learner's environment for stable variation to occur, if language acquisition operates along the lines of linear reward--penalty learning. This is a necessary but not a sufficient condition -- above we have seen, for example, that quasi-Babelian systems exhibit a phase transition between a phase in which stable variation occurs and one in which it does not occur (the most advantageous grammar instead claiming, eventually, all probability mass). Yet there is a kind of fatalism to these results: all symmetric systems, for instance, \emph{always} tend to an attractor which is a state of stable variation by Theorem \ref{thm:symmetric-stable}. It thus bears stressing that whenever stable variation occurs in these models, it is not due to extraneous factors such as social evaluations or population dynamics; stable variation follows from the nature of the LRP learning algorithm itself.

It may be instructive to consider this point in a little more detail. Thus consider step 4.b of Algorithm \ref{alg:LRP-nD}, corresponding to parsing failure. Here the algorithm tells us that whenever the grammar chosen by the learner, $G_k$, fails to parse a sentence, the learner updates the $k$th probability to become $\pi_k = (1-\gamma )\pi_k$. Thus the probability $\pi_k$ is diminished, and for all the grammar probabilities to keep summing to unity, it follows that some of the remaining probabilities need to be increased. From 4.b, we find that the learner actually updates every other probability $\pi_j$, $j \neq k$, to become
\begin{equation}\label{eq:nLRP-update}
  \frac{\gamma}{n-1} + (1-\gamma ) \pi_j.
\end{equation}
It is not difficult to check that these choices imply $\sum_i \pi_i = 1$, as desired. The consequences of choosing the update \eqref{eq:nLRP-update} over other possible choices, however, are nontrivial. Note that this manner of performing the update means that every grammar (apart from $G_k$, which failed) gets boosted by the same amount. This, then, means that the probability vector $\vec \pi = (\pi_1, \dots ,\pi_n)$ that describes the learner's grammar probabilities is shifted towards the centre $(1/n, \dots ,1/n)$ of the simplex at every occasion of parsing failure. When this mechanism is iterated over a diachronic sequence of learners, the effect gets amplified and, as we have seen, in some cases leads to diachronically stable variation. This observation also explains why the two-grammar version of the same algorithm behaves so differently: in this case, whenever one of the grammars fails to parse an input sentence, there is just one other grammar whose probability to boost. Consequently the probability vector describing the learner's state drifts towards dominance by this other grammar rather than towards a mixed state.

I would like to conclude by putting forward the following two conjectures, each supported by the special cases studied above but whose proofs have so far been elusive in the general case: (1) that \emph{any} $n$-grammar system with a proper advantage matrix has either $n$ rest points (the vertices) or $n+1$ rest points (the vertices plus one rest point in the interior of the simplex); and (2) that in \emph{any} proper system, if the interior rest point exists, it is necessarily asymptotically stable. If these results were to carry over to the NPL algorithm, too, the consequence would be clear: diachronic systems of learners operating on linear reward--penalty learning or variants thereof in multiple-grammar environments display a good deal of stable variation. Whether this is acceptable, or whether instead the above results call for a re-evaluation of the assumptions that underlie probabilistic language acquisition algorithms, needs to be answered by empirical work into the occurrence of stable variation in real-life language communities.

\section*{Acknowledgements}

Apart from the DiGS 18 pre-conference workshop on diachronic stability, portions of this work were presented at the 2016 Annual Meeting of the Linguistics Association of Great Britain (York, September 2016). I wish to thank both audiences, as well as Ricardo Bermúdez-Otero, George Walkden and an anonymous reviewer for feedback. All remaining errors and absurdities are, naturally, mine. The research here reported was made possible by generous financial support from Emil Aaltonen Foundation and The Ella and Georg Ehrnrooth Foundation.

\setlength{\bibsep}{0pt}

\end{document}